\documentclass[journal]{IEEEtran}

\usepackage{ifthen}
\newboolean{anon}
\newboolean{arxiv}
\setboolean{arxiv}{true}

\ifthenelse{\boolean{arxiv}}{
	\setboolean{anon}{false}
}{
	\setboolean{anon}{true}
}

\ifthenelse{\boolean{anon}}{
}{
}

\usepackage[numbers, sort, compress]{natbib}

\usepackage{pifont}

\ifCLASSINFOpdf
  \usepackage[pdftex]{graphicx}
\else
  \usepackage[dvips]{graphicx}
\fi

\usepackage{amsmath}
\usepackage{amssymb}
\usepackage{amsthm}
\usepackage{stackengine} 
\usepackage{aliascnt}
\usepackage[capitalise]{cleveref}
\usepackage{siunitx}

\usepackage{algorithm}
\usepackage{algpseudocode}
\usepackage{array}
\usepackage{aligned-overset}

\ifCLASSOPTIONcompsoc
    \usepackage[caption=false,font=normalsize,labelfont=sf,textfont=sf]{subfig}
\else
  \usepackage[caption=false,font=footnotesize]{subfig}
\fi

\usepackage{rotating}
\usepackage{multirow}
\usepackage{diagbox}

\usepackage{flushend}

\usepackage{url}
\usepackage{booktabs}
\usepackage{mathtools}

\DeclareMathOperator{\sd}{sd}

\DeclareMathOperator{\dist}{dist}

\DeclareMathOperator{\occ}{occ}
\DeclareMathOperator{\vel}{vel}

\newcommand{\Rbb}{\ensuremath{\mathbb{R}}}

\newcommand{\abs}[1]{\ensuremath{\left|{#1}\right|}}
\newcommand{\smid}{\ensuremath{ \, | \,}}
\newcommand{\tb}{\ensuremath{t_{M}}}

\newcommand{\td}{\ensuremath{t_{D}}}
\newcommand{\dt}{\ensuremath{\Delta t}}

\newcommand{\cb}{\ensuremath{\boldsymbol{c}}}
\newcommand{\db}{\ensuremath{\boldsymbol{d}}}

\newcommand{\nb}{\ensuremath{\boldsymbol{n}}}
\newcommand{\qb}{\ensuremath{\boldsymbol{q}}}
\newcommand{\rb}{\ensuremath{\boldsymbol{r}}}
\newcommand{\ub}{\ensuremath{\boldsymbol{u}}}
\newcommand{\wb}{\ensuremath{\boldsymbol{w}}}
\newcommand{\xb}{\ensuremath{\boldsymbol{x}}}

\newcommand{\zb}{\ensuremath{\boldsymbol{z}}}
\newcommand{\bom}{\ensuremath{\boldsymbol{\omega}}}

\newcommand{\Jb}{\ensuremath{\boldsymbol{J}}}
\newcommand{\Bb}{\ensuremath{\boldsymbol{B}}}
\newcommand{\Rb}{\ensuremath{\boldsymbol{R}}}
\newcommand{\Ib}{\ensuremath{\boldsymbol{I}}}
\newcommand{\Eb}{\ensuremath{\boldsymbol{E}}}
\newcommand{\Nb}{\ensuremath{\boldsymbol{N}}}

\newcommand{\Us}{\ensuremath{\mathcal{U}}}
\newcommand{\Xs}{\ensuremath{\mathcal{X}}}
\newcommand{\Ws}{\ensuremath{\mathcal{W}}}
\newcommand{\chix}[1]{\ensuremath{\chi_{\text{#1}}}}
\newcommand{\phix}[1]{\ensuremath{\Phi_{\text{#1}}}}
\newcommand{\uphix}[1]{\ensuremath{\ub_{\phix{#1}}}}

\newcommand{\pb}{\ensuremath{\boldsymbol{p}}}
\newcommand{\vb}{\ensuremath{\boldsymbol{v}}}
\newcommand{\rob}[1]{\ensuremath{{#1}^{r}}}
\newcommand{\robi}[1]{\ensuremath{\rob{{#1}}_{i}}}
\newcommand{\robin}[2]{\ensuremath{\rob{{#1}}_{i, #2}}}

\newcommand{\robjn}[2]{\ensuremath{\rob{{#1}}_{j, #2}}}
\newcommand{\robitop}[1]{\ensuremath{{#1}^{r \top}_{i}}}
\newcommand{\robl}[1]{\ensuremath{\rob{{#1}}_{l}}}

\newcommand{\Vrobi}[2]{\ensuremath{\robi{\mathcal{V}}(\left[#1, #2\right])}}
\newcommand{\Orobia}{\ensuremath{\robi{\mathcal{O}}}}
\newcommand{\Orobla}{\ensuremath{\robl{\mathcal{O}}}}
\newcommand{\robisimp}[1]{\ensuremath{{#1}_{i}}}

\newcommand{\hum}[1]{\ensuremath{{#1}^{h}}}
\newcommand{\humj}[1]{\ensuremath{\hum{{#1}}_{j}}}
\newcommand{\humjb}[1]{\ensuremath{\hum{{#1}}_{j'}}}
\newcommand{\humjc}[1]{\ensuremath{\hum{{#1}}_{\tilde{j}}}}

\newcommand{\Ohumja}{\ensuremath{\humj{\mathcal{O}}}}
\newcommand{\Ohumjb}{\ensuremath{\humjb{\mathcal{O}}}}

\newcommand{\Ohumj}[2]{\ensuremath{\humj{\mathcal{O}}(\left[#1, #2\right])}}
\newcommand{\env}[1]{\ensuremath{{#1}^{e}}}
\newcommand{\envk}[1]{\ensuremath{\env{{#1}}_{k}}}
\newcommand{\Oenvk}{\ensuremath{\envk{\mathcal{O}}}}
\newcommand{\body}[1]{\ensuremath{{#1}^{b}}}
\newcommand{\bodyn}[1]{\ensuremath{\body{{#1}}_{n}}}
\newcommand{\Obodyn}{\ensuremath{\bodyn{\mathcal{O}}}}

\newcommand{\Bs}{\ensuremath{\mathcal{B}}}
\newcommand{\Hb}{\ensuremath{\mathbb{H}}}
\newcommand{\Hs}{\ensuremath{\mathcal{H}}}
\newcommand{\Rs}{\ensuremath{\mathcal{R}}}
\newcommand{\Es}{\ensuremath{\mathcal{E}}}
\newcommand{\norm}[1]{\ensuremath{\|{#1}\|_2}}

\newcommand{\dqmax}[1]{\ensuremath{\hat{\dot{q}}_{#1}}}
\newcommand{\ddqmax}[1]{\ensuremath{\hat{\ddot{q}}_{#1}}}
\newcommand{\dddqmax}[1]{\ensuremath{\hat{\dddot{q}}_{#1}}}

\newcommand{\ommax}[1]{\ensuremath{\hat{\omega}_{#1}}}
\newcommand{\dommax}[1]{\ensuremath{\hat{\dot{\omega}}_{#1}}}
\newcommand{\ddommax}[1]{\ensuremath{\hat{\ddot{\omega}}_{#1}}}

\newtheorem{theorem}{Theorem}[section]

\newaliascnt{proposition}{theorem}
\newtheorem{proposition}[proposition]{Proposition}

\aliascntresetthe{proposition}

\crefname{section}{Sec.}{Sections}
\Crefname{section}{Sec.}{Sections}

\newcommand{\shield}{SARA shield}
\newcommand{\pfl}{power and force limiting}
\newcommand{\iso}{ISO/FDIS 10218-2:2024}

\makeatletter
\newcommand\footnoteref[1]{\protected@xdef\@thefnmark{\ref{#1}}\@footnotemark}
\newcommand*{\rom}[1]{\text{\expandafter\@slowromancap\romannumeral #1@}}
\makeatother

\newcommand*\circled[1]{
    \textcircled{\raisebox{-0.6pt}{\scriptsize{#1}}}
}

\begin{document}

\title{A General Safety Framework for Autonomous Manipulation in Human Environments}

\ifthenelse{\boolean{anon}}{
	\author{Anonymous authors%
		\thanks{$^{1}$Anonymous authors.
			{\tt\small anonymous authors email}}%
		\thanks{Manuscript received May 12, 2025; revised Month DD, YYYY.}}
}{
	\author{Jakob~Thumm,
		Julian Balletshofer,
		Leonardo Maglanoc,
		Luis Muschal,
		and~Matthias~Althoff%
		\thanks{$^{1}$The authors are with the School of Computation, Information and Technology, Technical University of Munich, 85748 Garching, Germany.
			{\tt\small \{jakob.thumm, julian.balletshofer, althoff\}@tum.de}}%
		\ifthenelse{\boolean{arxiv}}{}{
			\thanks{Manuscript received Month DD, YYYY; revised Month DD, YYYY.}
		}
		}
}

\ifthenelse{\boolean{arxiv}}{
	\markboth{}
}{
	\markboth{IEEE Transactions on Robotics,~Vol.~XX, No.~Y, Month~YYYY}
}
{\ifthenelse{\boolean{anon}}{Anonymous}{Thumm} \MakeLowercase{\textit{et al.}}: A General Safety Framework for Autonomous Manipulation in Human Environments.}

\maketitle

\begin{abstract}
Autonomous robots are projected to significantly augment the manual workforce, especially in repetitive and hazardous tasks.
For a successful deployment of such robots in human environments, it is crucial to guarantee human safety.
State-of-the-art approaches to ensure human safety are either too conservative to permit a natural human-robot collaboration or make strong assumptions that do not hold for autonomous robots, e.g., knowledge of a pre-defined trajectory.
Therefore, we propose the \underline{shield} for \underline{S}afe \underline{A}utonomous human-robot collaboration through \underline{R}eachability \underline{A}nalysis (\shield{}).
This novel \pfl{} framework provides formal safety guarantees for manipulation in human environments while realizing fast robot speeds.
As unconstrained contacts allow for significantly higher contact forces than constrained contacts (also known as clamping), we use reachability analysis to classify potential contacts by their type in a formally correct way.
For each contact type, we formally verify that the kinetic energy of the robot is below pain and injury thresholds for the respective human body part in contact.
Our experiments show that \shield{} satisfies the contact safety constraints while significantly improving the robot performance in comparison to state-of-the-art approaches.
\end{abstract}


\IEEEpeerreviewmaketitle

\section{Introduction}\label{sec:introduction}
Autonomous robots have already started to replace tedious, strenuous, and dangerous jobs~\cite{trevelyan_2016_RoboticsHazardous, smids_2020_RobotsWorkplace}, and are projected to make up a relevant part of the future workforce~\cite{inkwoodresearch_2021_SizeMarket, statista_2023_IndustryIndepth}.
We require these autonomous robots to safely collaborate with humans while exerting a high level of interactivity.
Previous approaches to ensure human safety in autonomous robotics often lack formal guarantees~\cite{suita_1995_FailuresafetyKyozon, kokkalis_2018_ApproachImplementing, svarny_2019_SafePhysical, lucci_2020_CombiningSpeed, ergun_2021_UnifiedPerceptiona, steinecker_2022_MeanReflected}.
Most formal approaches, however, make assumptions that do not hold for fully autonomous robots, e.g., pre-defined trajectories that prevent human clamping~\cite{haddadin_2012_MakingRobots, beckert_2017_OnlineVerification}, no contacts with the robot links~\cite{haddadin_2012_MakingRobots, liu_2021_OnlineVerification}, or stationary humans~\cite{haddadin_2012_MakingRobots}. 
Formal approaches that do not make these assumptions~\cite{lachner_2021_EnergyBudgets} are often too conservative to enable a seamless human-robot collaboration (HRC).

\iso{} defines two possible forms of contact: \emph{constrained contacts}, where the human is clamped by the robot link, and \emph{unconstrained contacts}, where the human can freely move.
Recent experiments have shown that the safe contact forces for unconstrained contacts are up to \num{20} times higher than for constrained contacts~\cite{kirschner_2021_ExperimentalAnalysis, kirschner_2024_UnconstrainedCollision}.
Each contact comprises two phases: in the \emph{transient} phase (first \SI{0.5}{\second}), peak contact forces have to be limited to minimize the risk of injury; in the subsequent \emph{quasi-static} phase, clamping forces have to remain below the pain onset thresholds.
The quasi-static phase is negligible for unconstrained contacts.

To mitigate the downsides of previous work, we propose the \underline{shield} for \underline{S}afe \underline{A}utonomous human-robot collaboarion through \underline{R}eachability \underline{A}nalysis (\shield{}).
Our novel \pfl{}~\cite{iso_2024_RoboticsSafety} approach formally guarantees human safety in robotic manipulation for (a) dynamically changing robot paths, (b) full human-robot contact, (c) arbitrary human motion, and (d) sharp robot geometries.
Our framework focuses on limiting all transient contact forces and can be extended with any contact reaction scheme mitigating quasi-static forces, e.g.,~\cite{vorndamme_2025_SafeRobot}.
\Cref{fig:overview} provides an overview of the methodology behind \shield{}.
In each control cycle, \shield{} detects all possible contacts between the robot and the human by determining the sets of all possible states the human and robot can occupy in a given time frame using reachability analysis.
Utilizing the same reachability analysis, we categorize the potential contacts into constrained and unconstrained types.
Once a contact is detected and classified, we reduce the speed of the robot so that the kinetic energy of the robot link in contact is below the corresponding thresholds of the detected contact type.
By differentiating the contacts by their type, we do not have to assume that all contacts are constrained and can verify against the less conservative unconstrained contacts in most situations.
As we base \shield{} on the previous work on provably safe manipulation for reinforcement learning-based HRC~\cite{thumm_2022_ProvablySafe}, we can directly deploy \shield{} on autonomous robots.

\begin{figure*}[t]
	\centering
	\includegraphics[width=0.99\textwidth]{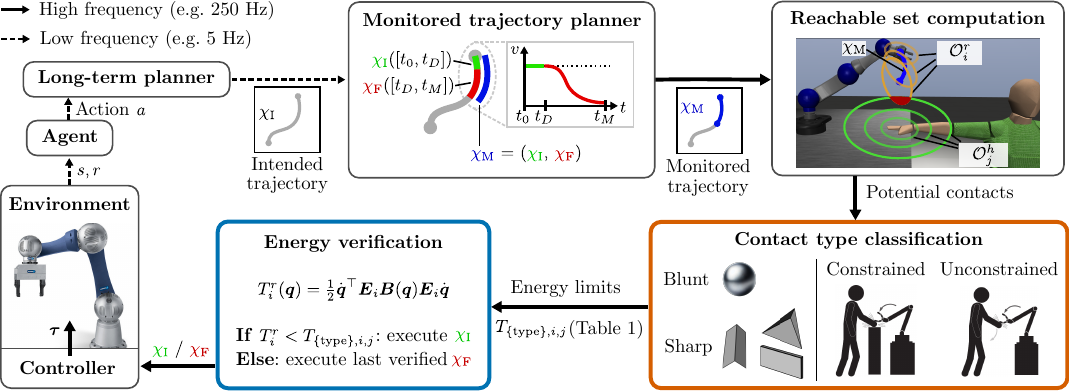}
	\caption{Overview of the SaRA-shield safety framework. 
    An action is translated into an intended trajectory. 
    In each control cycle, SaRA shield computes a monitored trajectory consisting of an intended trajectory followed by a failsafe trajectory. 
    We check the monitored trajectory for possible human contacts using reachability analysis. 
    When a potential contact with the human is detected, we classify the contact into the constrained or unconstrained type. 
    We verify if the kinetic energy of the robot upon contact is below pain and injury limits for the given contact type. 
    If the verification is successful, we execute the intended trajectory, and if it fails, we execute the last successfully verified failsafe trajectory.
    }
	\label{fig:overview}
\end{figure*}

Our core contributions are
\begin{itemize}
    \item a formal safety framework for manipulators in human environments including,
    \item a reachability-based clamping detection, and
    \item an injury prevention based on admissible contact energies.%
\end{itemize}

The main claims following our experimental results are that 
\begin{itemize}
    \item \shield{} significantly improves the robot performance in comparison to previous safe HRC approaches, and %
    \item the kinetic energy of a robot controlled with \shield{} during a contact with a human is below injury thresholds.
\end{itemize}

This article is structured as follows: \cref{sec:related_work} discusses the related work. 
We introduce the notation of our reachability analysis, the foundation of our safety mechanism in~\cref{sec:preliminaries}, and the problem statement in \cref{sec:problem_statement}. 
\cref{sec:safety_framework} provides a high-level overview of our proposed safety framework. 
We discuss our proposed contact classification in~\cref{sec:constrained_contact} and our energy-based verification in~\cref{sec:energy_constraints}. 
Finally, we evaluate our claims in~\cref{sec:experiments} and provide a conclusion in~\cref{sec:conclusion}.
\section{Related Work}\label{sec:related_work}
This section introduces the notion of (a) speed and separation monitoring, (b) \pfl{}, and (c) discusses relevant methodologies implementing them.
We subsequently discuss recent human injury and pain onset studies and their implications for safety in HRC.

\subsection{Speed and Separation Monitoring}\label{sec:related_work_ssm}
In speed and separation monitoring~\cite{iso_2024_RoboticsSafety}, the robot is stopped before a human can reach it.
Industrial applications usually implement speed and separation monitoring by enclosing the workspace of the robot with a light curtain that stops the robot once it is crossed.
The safety distance between the robot and the human can be decreased when a variable braking time of the robot and a limited human speed is assumed~\cite{marvel_2017_ImplementingSpeed, lacevic_2020_ExactSolution}.
When the human position can be continuously measured, the robot can resume its operation once the human retreats far enough from its workspace.
The authors in~\cite{beckert_2017_OnlineVerification, pereira_2018_OverapproximativeHuman, althoff_2019_EffortlessCreation, schepp_2022_SaRATool, thumm_2022_ProvablySafe} propose a formal speed and separation monitoring concept by over-approximating the reachable sets of the robot and surrounding humans.
These approaches allow the robot to continue operation even if the human is relatively close ($\approx$ \SI{0.2}{\meter}), assuming continuous measurements of the human pose are available.

\subsection{Power and Force Limiting}\label{sec:related_work_pfl}
In many HRC scenarios, we want to allow or even embrace contact with the human partner.
\iso{} permits such contacts using \pfl{}, where the robot is allowed to have a non-zero speed at contact as long as transient and quasi-static force, pressure, or energy limits are not violated.
It is important to note that not every body part can endure the same contact forces, as we discuss in~\cref{sec:related_work_human_injuries}.
Therefore, methods that do not assume knowledge about the human pose have to assume the worst case and tend to be conservative in their maximal allowed robot speed.

First approaches implementing \pfl{} proposed to stop the robot once a contact is detected based on an external force estimation from motor current differences~\cite{suita_1995_FailuresafetyKyozon, kokkalis_2018_ApproachImplementing, vorndamme_2025_SafeRobot}.
However, as shown in~\cite{haddadin_2009_RequirementsSafe, kirschner_2021_ExperimentalAnalysis}, the reaction time of such approaches is too slow to significantly mitigate the initial transient contact force spikes. 
Hence, these methods are mostly relevant for reducing contact forces in the quasi-static phase.

\iso{} defines a basic \pfl{} method that approximates the robot mass acting on the contact as half of the total robot mass and calculates the maximal contact forces based on the kinetic energy of the robot.
Although it is an easy-to-use model, \citet{kirschner_2021_NotionCorrect} have shown that this robot mass calculation is incorrect, and the reflected robot mass~\cite{khatib_1995_InertialProperties} should be used instead.
\iso{} also specifies a reduced-speed mode, where any contact is considered safe while the robot speed is below \SI{0.25}{\meter \per \second}.
However, this approach violates contact force thresholds for large robot masses~\cite{kirschner_2021_ExperimentalAnalysis} and only permits blunt robot geometries.
Hence, we argue that the simple safety measures outlined in \iso{} do not provide formal safety guarantees for all collaborative robots.

\citet{haddadin_2012_MakingRobots} propose a \pfl{} approach based on the reflected robot mass, where they assume (a) knowledge of the human pose, (b) that the human is stationary, and (c) the critical contact point is known.
Using these assumptions, they determine the contact normal between the robot and the human, along which they calculate the reflected robot mass.
\citet{haddadin_2012_MakingRobots} are also the first to propose the direct use of an injury database to determine a so-called safe motion unit, which returns the maximal allowed speed of the robot given the reflected robot mass.
They argue that this leads to a more accurate representation of the injury cases than the usually used transient force limits.
Finally, they scale the speed of the robot to the maximal allowed value in each control cycle.
The major downside of this approach is the assumed knowledge of the contact normal, which does not hold in reality.

To overcome this limitation, \citet{steinecker_2022_MeanReflected} propose to use the mean reflected mass over all possible contact normals as a mass metric. 
However, this approximation cannot provide formal safety guarantees.
Instead of verifying human safety based on the reflected mass, recent works~\cite{lachner_2021_EnergyBudgets, benzi_2023_EnergyTankbased} propose to use an energy tank-based control to limit the total kinetic energy of the robot. 
However, these approaches do not take the human pose and contact type into account and thus have conservative energy thresholds.

A downside of most \pfl{} approaches~\cite{suita_1995_FailuresafetyKyozon, haddadin_2012_MakingRobots, kokkalis_2018_ApproachImplementing, sloth_2018_ComputationSafe, aivaliotis_2019_PowerForce, lachner_2021_EnergyBudgets, steinecker_2022_MeanReflected, golshani_2023_ControlDesign, hamad_2023_ModularizeconquerGeneralized, pupa_2024_TimeOptimalEnergy} is the assumption that a contact is always possible, leading to a reduced speed of the robot even if the human is far away.
To overcome this, the authors in~\cite{svarny_2019_SafePhysical, lucci_2020_CombiningSpeed} propose to switch between \pfl{} and speed and separation monitoring based on the distance between the human and the robot.
The approach in~\cite{beckert_2017_OnlineVerification} extends this idea by dynamically reducing the speed of the robot if a possible contact is detected using reachability analysis.
\citet{liu_2021_OnlineVerification} further improve on this by replacing the otherwise fixed maximal contact speed with a dynamic one that is verified using model conformance checking.
Despite promising results, their approach is only applicable to end-effector contacts and requires human contact measurements for each new robot, body part, and a wide range of participants.

\subsection{Human Injuries}\label{sec:related_work_human_injuries}
Human pain onset and skin damage depend on the impact energy, contact type, and contact shape~\cite{eisenmenger_2004_SpitzeScharfe, povse_2010_CorrelationImpactenergy, haddadin_2012_MakingRobots, behrens_2014_StudyMeaningful, behrens_2023_StatisticalModel}.
The most researched contact situations are constrained contacts with blunt impactors that evaluate the subjective onset of pain~\cite{yamada_1996_FailsafeHuman, behrens_2014_StudyMeaningful, behrens_2023_StatisticalModel, asad_2023_BiomechanicalModeling}.
\iso{} defines force, pressure, and energy limits for constrained contacts, which seem to be confirmed by the aforementioned experimental studies.
However, there remains a large variance in the study outcomes.
When it comes to sharp or edged contacts in constrained contacts, the ex-vitro study of~\citet{kirschner_2024_SafeRobot} shows that the ISO-defined limits are too high to prevent skin damage.
A subsequent study for unconstrained contacts~\cite{kirschner_2024_UnconstrainedCollision} showed that critical skin injuries occur at \num{5}--\num{20} times higher impact velocities than in the constrained case.
These studies show that sharp contact geometries should be taken into account and that if no contact classification is performed, safety mechanisms have to assume the more conservative constrained contact type at all times.

\section{Preliminaries and Problem Statement}\label{sec:preliminaries}
This section introduces reachability analysis, the safety shield concept we are building upon, our problem statement, and the assumptions we make.

\subsection{Reachability Analysis}
We adopt a reachability analysis approach for safety verification.
Our considered systems follow the dynamics $\dot{\xb}(t)=f(\xb(t), \ub(t), \wb(t))$, with time $t \in \Rbb_0^+$, state $\xb(t) \in \Rbb^n$, input $\ub(t) \in \Us \subset \Rbb^m$, and disturbance $\wb(t) \in \Ws \subset \Rbb^p$, where $\mathcal{U}$ and $\mathcal{W}$ are compact sets.
For the given initial state $\xb_0$, input trajectory $\ub(\cdot)$\footnote{Note that $\ub(\cdot)$ refers to the input trajectory and $\ub(t)$ refers to the input at time $t$.}, and disturbance trajectory $\wb(\cdot)$, we denote the unique solution of the system dynamics as {$\chi\left(t; \xb_0, \ub(\cdot), \wb(\cdot)\right) \in \Rbb^n$}.
The exact reachable set of the system at time $t$ can be defined for a set of initial states $\Xs_0$ as
\begin{align}
    \mathcal{X}^e(t) = & \big\{ \chi\left(t; \xb_0, \ub(\cdot), \wb(\cdot)\right) \mid \\
    & \quad \xb_0 \in \Xs_0, \forall \tau \in [t_0, t]: \ub(\tau) \in \Us, \wb(\tau) \in \Ws \big\} \nonumber.
\end{align}
In general, the exact reachable set $\mathcal{X}^e(t)$ cannot be computed~\cite{gan_2018_ReachabilityAnalysis}, so we compute over-approximations for particular points in time {$\mathcal{X}(t) \supseteq \mathcal{X}^e(t)$}, which we simply refer to as the reachable set for simplicity.
All states reachable during the time interval $[t_a, t_b]$ are denoted by {$\mathcal{X}\left(\left[t_a, t_b\right]\right)=\bigcup_{t \in\left[t_a, t_b\right]} \mathcal{X}(t)$}.

Let us introduce the operator $\occ()$ returning the possible occupancies of a system for a given reachable set.
We denote the set of points that a system can occupy in Euclidean space at time $t$ and the time interval $[t_a, t_b]$ as {$\mathcal{O}(t) = \occ(\mathcal{X}(t))$} and {$\mathcal{O}(\left[t_a, t_b\right]) = \occ(\mathcal{X}(\left[t_a, t_b\right]))$}, respectively, and further refer to them as reachable occupancies.
Similarly, we define the reachable velocity as the set of linear velocities that any point of the system can have in the time interval $\left[t_a, t_b\right]$ as {$\mathcal{V}(\left[t_a, t_b\right]) = \vel(\mathcal{X}(\left[t_a, t_b\right]))$}, where $\vel : \mathcal{P}(\Rbb^n) \rightarrow \mathcal{P}(\Rbb^3)$ extracts all possible linear velocities from a reachable set and $\mathcal{P}$ is the power set.
We denote the set of all human body parts in the scene as $\Hs = \left\{h_1, \dots, h_M\right\}$ and the set of robot links as $\Rs = \left\{r_1, \dots, r_N\right\}$.
A point in Euclidean space is denoted as $\pb$,
and a capsule is defined by two points $\pb_1 \in \Rbb^3$ and $\pb_2 \in \Rbb^3$ and a radius $r \in \Rbb^+$ with {$\mathcal{C}(\pb_1, \pb_2, r) = \overline{\pb_1 \pb_2} \oplus \Bs(\boldsymbol{0}, r)$}, where {$\overline{\pb_1 \pb_2}$} is the line segment between $\pb_1$ and $\pb_2$, $\Bs(\cb, r)$ is a ball with center $\cb$ and radius $r$, the symbol $\oplus$ represents the Minkowski sum, and $\boldsymbol{0}$ is the vector of zeros.
A polytope is defined in the halfspace representation as ${\Hb(\Nb, \db) = \left\{\pb \in \Rbb^3 \smid \Nb \pb \le \db, \Nb \in \Rbb^{H \times 3}, \db \in \Rbb^{H}\right\}}$ {\cite[Sec. 2.5]{althoff_2015_IntroductionCORA}}.
The operations $\Nb[h]$ and $\db[h]$ index the $h$-th row vector of the matrix $\Nb$ and the $h$-th entry of the vector $\db$, respectively.
Finally, we define the signed distance function between a point $\pb$ and a set $\mathcal{S}$ as 
\begin{align}
	\sd(\pb, \mathcal{S}) = 
	\begin{cases}
		- \underset{\pb_s \in \partial \mathcal{S}}{\inf} \norm{\pb - \pb_s}, \ \pb \in \mathcal{S} \\
		\underset{\pb_s \in \partial \mathcal{S}}{\inf} \norm{\pb - \pb_s}, \ \pb \notin \mathcal{S} \, .
	\end{cases}
\end{align}

\subsection{Safety Shield}\label{sec:safety_shield}
To ensure the safety of robots controlled by autonomous agents, we adapt the safety shield for manipulators proposed in~\cite{thumm_2022_ProvablySafe}.
The safety shield relies on the existence of a set of invariably safe states that can be reached from any state.
\citet{thumm_2022_ProvablySafe} define the set of invariably safe states as a fully stopped robot for manipulators in accordance with the speed and separation monitoring formulation in \iso{}.
\citet{thumm_2022_ProvablySafe} further demonstrate that a safety shield for autonomous agents is less conservative when it operates on a higher frequency than the output frequency of the agent.
Therefore, they convert each action of the agent $\boldsymbol{a}$ into a set of goal states $\mathcal{X}_g(\boldsymbol{a})$, and aim to steer the robot towards this set for $L$ time steps, while performing a safety shield update at every time step.

In each safety shield update, \citet{thumm_2022_ProvablySafe} compute an intended and a failsafe trajectory.
Without loss of generality, we always reset the clock to $t_0=0$ at each time step.
The intended trajectory $\chix{I}$ steers the robot towards $\mathcal{X}_g(\boldsymbol{a})$ using the output $\uphix{I}(\xb, \mathcal{X}_g(\boldsymbol{a}))$ of the intended controller for $D$ time steps; in this work, we use $D=1$.
The failsafe trajectory $\chix{F}$ steers the robot to an invariably safe state using the output $\uphix{F}(\xb, \mathcal{X}_g(\boldsymbol{a}))$ of the failsafe controller for $k_\text{F}$ time steps ending at time $\tb$.
Their failsafe controller performs a braking maneuver that is path-consistent with the intended trajectory, which prevents the robot from deviating from the set of goal states.
The shield appends a full failsafe trajectory to an intended trajectory to form a so-called monitored trajectory 
\begin{align}\label{eq:shielded_trajectory}
    \chix{M} = 
    \begin{cases}
        \chix{I}\left(t; \xb_0, \uphix{I}(\cdot), \wb(\cdot)\right), &t \in [t_0, \td]\\
        \chix{F}\left(t; \xb_D, \uphix{F}(\cdot), \wb(\cdot)\right), & t \in [\td, \tb].
    \end{cases}
\end{align}
If the safety of the monitored trajectory is verified successfully, the robot executes the intended trajectory, and otherwise the failsafe trajectory from the last successfully verified monitored trajectory.
When assuming that the robot starts from an invariably safe state, safety is ensured indefinitely by induction~{\cite[p.~5]{althoff_2019_EffortlessCreation}}.

The authors in~\cite{beckert_2017_OnlineVerification, thumm_2022_ProvablySafe} verify the monitored trajectory against the specification that a contact between the human and robot should be prevented entirely.
For this, they compute the reachable occupancies of all human body parts $h_j \in \Hs$ in the time interval $\left[t_0, \tb\right]$ denoted as $\Ohumj{t_0}{\tb}$ assuming that the human inputs stem from a compact set $\humj{\ub}(t) \in \humj{\mathcal{U}}$, e.g., each human body part can have a maximal speed of \SI{1.6}{\meter \per \second} as defined in~\cite{iso_2010_SafetyMachinery}. 
They further compute the reachable occupancies of all robot links $r_i \in \Rs$ following the monitored trajectory as $\Orobia(\left[t_0, \tb\right])$.

The verification evaluates the safety of the monitored trajectory for all time intervals, robot links, and human body parts, defining the constraint
\begin{align}\label{eq:combined_constraint}
	&c_{\text{safe}}(\chix{M}) \coloneqq \bigwedge_{a=0}^{M-1} \bigwedge_{r_i \in \Rs} \bigwedge_{h_j \in \Hs} c_{\text{safe}, i, j}([t_a, t_{a+1}]) \, .
\end{align}
\citet{thumm_2022_ProvablySafe} consider the movement of a robot link $r_i$ to be safe in the time interval $[t_a, t_b]$ with respect to the body part $h_j$ if their reachable occupancies do not intersect
\begin{align}
    c_{\text{safe}, i, j}([t_a, t_b]) & \: = \bar{c}_{\text{contact}, i, j}([t_a, t_b]) \, , \nonumber \\
    \bar{c}_{\text{contact}, i, j}([t_a, t_b]) &\coloneqq  \Ohumja([t_a, t_b]) \cap \Orobia([t_a, t_b]) = \varnothing \, ,
\end{align}
where $\bar{c}$ is the negation of the predicate $c$.
\citet{schepp_2022_SaRATool} enables an efficient computation of these reachable occupancies and intersection checks, using capsules as reachable occupancies of the robot links and human body parts.

\subsection{Problem Statement}\label{sec:problem_statement}
We aim to derive a safety framework that allows the robot to come into contact with the human as long as it is slow enough to prevent injuries.
For this, we want to operate the robot at high speed as long as possible and only reduce the speed if necessary.
\iso{} provides three types of contact constraints: force, pressure, and energy, from which we are using the energy constraints for our problem statement.
We aim to minimize a general objective $J_g$ while ensuring that if a contact with the human is possible, the robot energy is below human pain and injury limits with
\begin{subequations}
\begin{align}
    &\qquad \min_{\uphix{}} \  J_g\label{eq:objective_a}\\
    &\text{subject to } \chi\left(t_g; \xb_0, \uphix{}(\cdot), \wb(\cdot)\right) \in \mathcal{X}_g(\boldsymbol{a}), \label{eq:objective_b}\\
    & \quad \forall r_i \in \Rs, h_j \in \Hs, t\in \left[t_0, t_g\right]: \\
    & \qquad \Ohumja(t) \cap \Orobia(t) = \varnothing \lor \robi{T}(t) \leq T_{i, j}, \label{eq:objective_c}\\
    & \quad \humj{\ub}(\cdot) \in \humj{\mathcal{U}}, \uphix{}(\cdot) \in \mathcal{U}, \wb(\cdot) \in \mathcal{W},\label{eq:objective_e}
\end{align}
\end{subequations}
where $t_g$ is the time the robot needs to reach the set of goal states, $\robi{T}(t)$ describes the maximal reachable kinetic energy that robot link~$r_i$ induces into the human contact at time $t$, and $T_{i, j}$ denotes the energy threshold at which a rigid body with the geometry of robot link $r_i$ is unlikely to cause an injury for the human body part $h_j$.
In our experiments, we always set the objective to minimizing the time to reach the goal $J_g=t_g$.
However, other objectives such as additionally minimizing the jerk or motor currents would be possible.
Note that the contact energy constraint can be converted into force or pressure constraints if we assume knowledge about the stiffness and damping properties of the human skin as detailed in Annex M of \iso{}. 

\subsection{Assumptions}\label{sec:assumptions}
Similar to \citet{thumm_2022_ProvablySafe}, we make three key assumptions in this work.
First, we assume that the human pose is measurable with a bounded measurement error of $\delta_{\text{meas}}$ and a bounded time delay of $\Delta t_{\text{meas}}$.
In our experiments, we achieve this by using a motion-tracking system, but our formulation can also be used with laser scanners~\cite{kumar_2019_SpeedSeparation} or depth cameras~\cite{himmelsbach_2018_SafeSpeed, rosenstrauch_2018_HumanRobot, yang_2022_DynamicSpeed, fujii_2024_RealtimeBatched}. 
The measurement error and delay are incorporated into the human reachable set computation using the set-based prediction in~\cite{schepp_2022_SaRATool}.
Second, we assume that the speed of any human body part never exceeds \SI{1.6}{\meter \per \second} when approaching the robot, as defined in DIN EN ISO 13855:2010.

Third, we assume that the intended and failsafe controller track the desired trajectory with a bounded tracking error, which is incorporated in the bounded system noise $\mathcal{W}$ of the robot reachable occupancy computation $\rob{\mathcal{O}}$~\cite{schepp_2022_SaRATool}. 
In our experiments, we are using a PD-controller for trajectory tracking, but prior work has effectively demonstrated bounded tracking error guarantees for a trajectory tracking controller based on model conformance checking~\cite{liu_2023_GuaranteesReal}.
Given the controller, for all joints $i=1, \dots, N$, we assume knowledge of the maximal possible joint velocities $\dqmax{i}$, acceleration $\ddqmax{i}$, and jerk $\dddqmax{i}$ over all trajectories and for all times $t$ with
\begin{align}
    \dqmax{i} &= \underset{t \leq \tau \leq t+\dt}{\sup} \abs{\dot{q}_i(\tau)} \\
    \ddqmax{i} &= \underset{t \leq \tau \leq t+\dt}{\sup} \abs{\ddot{q}_i(\tau)} \\
    \dddqmax{i} &= \underset{t \leq \tau \leq t+\dt}{\sup} \abs{\dddot{q}_i(\tau)} \, ,
\end{align}
where $\dt = t_1 - t_0$.

We additionally assume that through knowledge of the robot kinematics, we can determine the linear and angular velocity and acceleration of a point on the robot at a given time with bounded errors of $\norm{\wb_v} = \norm{\dot{\pb} - \dot{\pb}^e} \le w_{v, \, \text{max}}$, $\norm{\wb_{\omega}} = \norm{{\bom} - {\bom^e}} \le w_{\omega, \, \text{max}}$, $\norm{\wb_a} = \norm{\ddot{\pb} - \ddot{\pb}^e} \le w_{a, \, \text{max}}$, and $\norm{\wb_{\dot{\omega}}} = \norm{\dot{\bom} - \dot{\bom}^e} \le w_{\dot{\omega}, \, \text{max}}$, respectively, where $(\cdot)^e$ indicates the true value.
We can, e.g., use reachset conformance checking~\cite{liu_2018_ReachsetConformance} to establish these error bounds.

Similarly to most \pfl{} approaches~\cite{haddadin_2012_MakingRobots, beckert_2017_OnlineVerification, sloth_2018_ComputationSafe, svarny_2019_SafePhysical, lucci_2020_CombiningSpeed, lachner_2021_EnergyBudgets, liu_2021_OnlineVerification, steinecker_2022_MeanReflected, golshani_2023_ControlDesign, hamad_2023_ModularizeconquerGeneralized, pupa_2024_TimeOptimalEnergy}, we also assume that the human does not actively move into hurtful contact with the robot so that the speed difference between the human and the robot along the contact normal is always smaller or equal to the robot speed.

For classifying contacts, we additionally assume there is an over-approximating time-independent polytope representation of all environment elements that can cause clamping.
This information is typically available in industrial applications or can be obtained using laser scanning~\cite{bi_2021_SurveyLowCost}. 
\section{Safety framework for human-robot interaction} \label{sec:safety_framework}

\begin{figure}
    \includegraphics[width=\columnwidth]{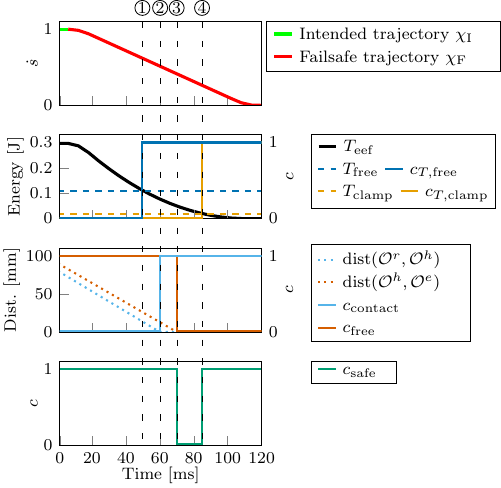}
    \caption{Example evolution of the constraints in~\eqref{eq:complete_constraint} over a monitored trajectory. 
    \circled{1} The kinetic energy of the robot end effector $T_{\text{eef}}$ drops below the threshold for unconstrained contacts $T_{\text{free}}$, so $c_{T, \text{free}} = \mathrm{true}$.
    \circled{2} The reachable occupancies of the robot and the human intersect ($\text{dist}(\rob{\mathcal{O}}, \rob{\mathcal{O}}) \le 0$), indicating that an unconstrained contact is possible, but the robot energy is safe.
    \circled{3} The human occupancy intersects with the environment ($\text{dist}(\hum{\mathcal{O}}, \env{\mathcal{O}}) \le 0$) and the robot ($\text{dist}(\rob{\mathcal{O}}, \hum{\mathcal{O}}) \le 0$), indicating that a constrained contact is possible.
    Since the robot energy is violating the constrained contact energy threshold $T_{\text{eef}} \ge T_{\text{clamp}}$, the trajectory is unsafe $c_{\text{safe}} = \mathrm{false}$.
    \circled{4} The robot energy drops below the constrained contact threshold, but once one time interval of the trajectory is unsafe, the entire trajectory is unsafe.
    }
    \label{fig:constraint_evolution}
\end{figure}

As highlighted in~\cref{sec:related_work_human_injuries}, unconstrained contacts permit a significantly higher transient contact force than constrained contacts~\cite{kirschner_2024_SafeRobot, kirschner_2024_UnconstrainedCollision}.
Therefore, we propose for the first time to classify possible contacts into constrained and unconstrained types using reachability analysis.
We then constrain the kinetic energy of the robot links at the contact points to be below the pain and injury limits for the respective contact type.

For the verification in~\eqref{eq:combined_constraint}, we define the safety constraint for a given robot link~$r_i$, human body part~$h_j$, and time interval $[t_a, t_b]$ as
\begin{align}\label{eq:complete_constraint}
c_{\text{safe}, i, j}([t_a, t_b]) \coloneqq \, &\bar{c}_{\text{contact}, i, j}([t_a, t_b]) \, \lor \\ & \left(c_{\text{free}, i, j}([t_a, t_b]) \land c_{T, \, \text{free}, i, j} \right) \lor \nonumber \\
& \left(\bar{c}_{\text{free}, i, j}([t_a, t_b]) \land c_{T, \, \text{clamp}, i, j} \right) \, , \nonumber
\end{align}
where 
\begin{itemize}
	\item $\bar{c}_{\text{contact}, i, j}$ evaluates if there is no contact between the robot link~$r_i$ and the human body part~$h_j$,
	\item $c_{\text{free}, i, j}$ evaluates if the contact between link $r_i$ and human body part $h_j$ is an unconstrained contact,
	\item $c_{T, \, \text{free}, i, j}$ evaluates if the kinetic energy of link~$r_i$ is below the injury threshold of body part~$h_j$ for unconstrained contacts, and
	\item $c_{T, \, \text{clamp}, i, j}$ evaluates if the kinetic energy of link~$r_i$ is below the pain and injury threshold of body part~$h_j$ for constrained contacts.
\end{itemize}
A typical evolution of these constraints over the length of a monitored trajectory is depicted in~\cref{fig:constraint_evolution}.
We derive $c_{\text{free}, i, j}$ in~\cref{sec:constrained_contact} and $c_{T, \, \text{\{type\}}, i, j}$ in~\cref{sec:energy_constraints}.

\section{Constrained contact detection}\label{sec:constrained_contact}
In this section, we present how we can use reachability analysis to detect potential constrained contacts.
In general, the human can be clamped between the robot and the static environment or in between two robot links.
We will refer to these two cases as environmentally constrained contacts (ECCs) and self-constrained contacts (SCCs).
Hence, we can ensure that a contact between link~$r_i$ and body part~$h_j$ in the time interval $[t_a, t_b]$ is unconstrained if
\begin{align}\label{eq:constrained_contact_constraint}
    c_{\text{free}, i, j}([t_a, t_b]) \coloneqq \bar{c}_{\text{ECC}, i, j}([t_a, t_b]) \land \bar{c}_{\text{SCC}, i, j}([t_a, t_b]) \, ,
\end{align}
i.e., the contact is neither an ECC ($\bar{c}_{\text{ECC}, i, j}$) nor an SCC ($\bar{c}_{\text{SCC}, i, j}$).
In the following subsections, we derive the general constraints $\bar{c}_{\text{ECC}, i, j}$ and $\bar{c}_{\text{SCC}, i, j}$, formulate a series of relaxations of those constraints, and discuss how we handle clamping of multiple human body parts.
\Cref{fig:clamping-constraints} visualizes these concepts for the example of an ECC.

\begin{figure}[t]
    \centering
    \includegraphics[width=\columnwidth]{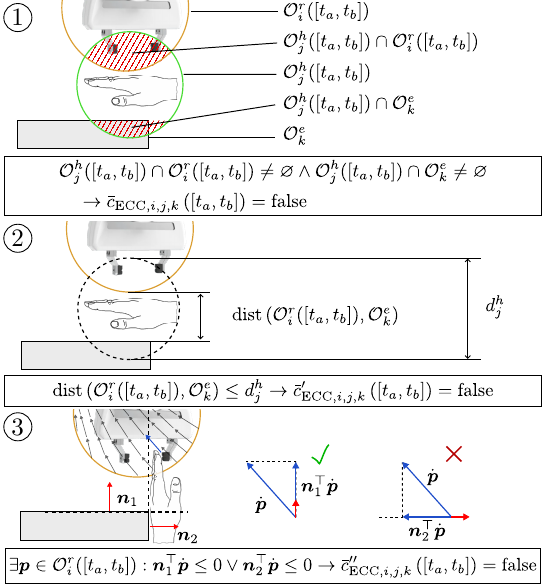}
    \caption{Example of clamping constraints:
    \circled{1} Since the reachable occupancy of the human intersects the ones of the robot and environment element, clamping is possible. 
    \circled{2} The smallest distance of the robot to the environment element is smaller than the human diameter, so clamping is possible. 
    \circled{3} There exists a point in the reachable occupancy of the robot that is moving towards the vertical face of the environment element, which could cause clamping as illustrated.}
    \label{fig:clamping-constraints}
\end{figure}

\subsection{General ECC and SCC Constraints} 
We model the environment as a set of large unmovable elements $\Es = \left\{e_1, \dots, e_K\right\}$, such as walls, the floor, or a table, with fixed occupancies $\Oenvk, e_k \in \Es$.
We define a trajectory $\chix{M}$ to be free of ECCs between a robot link~$r_i$ and a human body part~$h_j$ in a time interval $[t_a, t_b]$ if:
\begin{align}\label{eq:ecc}
	\bar{c}_{\text{ECC}, i, j}([t_a, t_b]) &\coloneqq \bigwedge_{e_k \in \Es} \bar{c}_{\text{ECC}, i, j, k}\left([t_a, t_b]\right)\\
    \bar{c}_{\text{ECC}, i, j, k}\left([t_a, t_b]\right) &\coloneqq  \Ohumja([t_a, t_b]) \cap \Oenvk = \varnothing \, , \label{eq:ecc_single}
\end{align}
i.e., there cannot be an ECC with the environment element $e_k$ if the reachable occupancy of the human body part does not intersect with it. 

For SCCs, we exclude a potential clamping of a body part~$h_j$ between two robot links $r_i, r_l \in \Rs, r_i \neq r_l$ in a time interval $[t_a, t_b]$ if:
\begin{align}\label{eq:scc_single}
    \bar{c}_{\text{SCC}, i, j, l}\left([t_a, t_b]\right) &\coloneqq \bar{c}_{\text{contact}, i, j}([t_a, t_b]) \lor \bar{c}_{\text{contact}, l, j}([t_a, t_b]) \, ,
\end{align}
i.e., there is no SCC if the human occupancy does not intersect with two different robot links.
Based on~\eqref{eq:scc_single}, we define the SCC constraint for a link~$r_i$ and a body part~$h_j$ as
\begin{align}\label{eq:scc}
	\bar{c}_{\text{SCC}, i, j}([t_a, t_b]) &\coloneqq \bigwedge_{r_l \in \Rs \setminus r_i} \bar{c}_{\text{SCC}, i, j, l}\left([t_a, t_b]\right) \, .
\end{align}

\subsection{Reducing Conservatism}\label{sec:reducing_conservatism}
The constraints in \eqref{eq:ecc} and \eqref{eq:scc} are quite conservative and would be often violated in close HRC.
Therefore, we subsequently propose three adaptions that reduce the conservatism of these constraints. %

\subsubsection{Human Diameter}\label{sec:human_diameter}
A natural condition that is required for constrained contacts is that the shortest distance between the robot link $r_i$ and environment element $e_k$ (ECC) or the two robot links $r_i$ and $r_l$ (SCC) must be smaller or equal to the maximal diameter of the human body part that could be clamped.
Therefore, we extend the constraints in~\eqref{eq:ecc_single} to 
\begin{align}
    \bar{c}_{\text{ECC}, i, j, k}'\left([t_a, t_b]\right) \coloneqq &\bar{c}_{\text{ECC}, i, j, k} \left([t_a, t_b]\right) \, \lor  \label{eq:ecc_single_2}\\ 
    & \  \dist\left(\Orobia([t_a, t_b]), \Oenvk\right) > \humj{d} \, , \nonumber
\end{align}
and the one in~\eqref{eq:scc_single} to
\begin{align}
    &\bar{c}_{\text{SCC}, i, j, l}' \left([t_a, t_b]\right) \coloneqq \bar{c}_{\text{SCC}, i, j, l} \left([t_a, t_b]\right) \, \lor \label{eq:scc_single_2}\\
    & \qquad \qquad \qquad \dist\left(\Orobia([t_a, t_b]), \Orobla([t_a, t_b])\right) > \humj{d} \, , \nonumber
\end{align}
where $\dist$ is a distance function that measures the smallest distance between two sets and $\humj{d}$ is the diameter of the human body part.
In Euclidean space, the distance function can be described as 
\begin{align}\label{eq:dist_1n}
    \dist(\mathcal{S}_1, \mathcal{S}_2) &= \underset{\pb_1, \, \pb_2}{\inf} \norm{\pb_1 - \pb_2}, \pb_1 \in \mathcal{S}_1, \pb_2 \in \mathcal{S}_2 \, ,
\end{align}
which we can efficiently compute for capsule representations~\cite[Ch. 10]{schneider_2003_GeometricTools}.
DIN 33402-2:2020-12 defines the \num{5}, \num{50}, and \num{95} percentile of human body part dimensions.
Hence, we use the \num{95} percentile measurements as an upper bound for the diameters of human body parts and assume that the system is informed if a user with a particularly large body size operates next to the robot.

\subsubsection{Robot Velocity}\label{sec:robot_velocity}
Clamping of a human body part between a point on the robot and a rigid body is only possible if the robot point lies outside of the rigid body and is moving towards it, i.e., the signed distance between the point and the rigid body is positive and increasing.
Therefore, we can exlude clamping for a robot link $r_i$ and a rigid body $b_n$ with occupancy $\Obodyn$ in the time interval $[t_a, t_b]$ if
\begin{align}\label{eq:increasing_distance_condition}
    &\forall t \in [t_a, t_b]: \forall \pb \in \Orobia(t): \,  \\
    & \qquad \qquad \pb \notin \Obodyn(t) \land \frac{d}{d t} \sd(\pb, \Obodyn(t)) \ge 0 \, ,\nonumber
\end{align}
i.e., the signed distance of any point in $\Orobia(t)$ to $\Obodyn(t)$ is increasing if that point lies outside of $\Obodyn(t)$.

In our case, the rigid body can be another robot link (SCC) or an environment element (ECC), so $b_n \in (\Rs \setminus r_i) \cup \Es$.
In the case of SCCs, we express the velocity of robot link $r_i$ in the frame of the other robot link $r_l$, and over-approximate the reachable occupancy of robot link $r_l$ with a polytope.
As such, we treat the robot link $r_l$ as a static polytope and derive the formal verification for an arbitrary rigid body $b_n$.

To verify~\eqref{eq:increasing_distance_condition}, it is insufficient to check if the smallest distance between the two sets is increasing, i.e., $\dist\left(\Orobia(t_b), \Obodyn(t_b)\right) > \dist\left(\Orobia(t_a), \Obodyn(t_a)\right)$, as any point on the robot can cause clamping, not just the closest point to the rigid body.
See part \circled{3} of~\cref{fig:clamping-constraints} as an example, where the closest point is moving away from the environment element, but clamping is still possible.
Instead, we propose to verify~\eqref{eq:increasing_distance_condition} for constrained contacts involving rigid bodies which occupancy is enclosed by one or several polytopes.

\begin{proposition}\label{prop:normal_velocity}
A constrained contact involving the robot link~$r_i$ and a rigid body $b_n$ with enclosing polytope $\Hb(\Nb, \db), \, \Nb \in \mathbb{R}^{H \times 3}$ in the time interval $[t_a, t_b]$ can be excluded if
\begin{align}\label{eq:velocity_constraint}
	&\Orobia([t_a, t_b]) \cap \Hb(\Nb, \db) = \varnothing \, \land \, \\
	& \qquad \forall \nb \in \mathcal{N}_{i, n}:	\forall \dot{\pb} \in \Vrobi{t_a}{t_b} : \nb^\top \dot{\pb}  \ge 0 \, , \nonumber
\end{align}
where $\mathcal{N}_{i, n}$ is the set of normal vectors defining the clamping-relevant halfspaces of the enclosing polytope with
\begin{align}
	\mathcal{N}_{i, n} = \{ \Nb[h]^\top &\smid h = 1, \dots, H: \exists \pb \in \Orobia([t_a, t_b]) : \\
	& \quad \Nb[h] \pb > \db[h] \} \, . \nonumber
\end{align}
\end{proposition}
\begin{proof}
	We prove $\eqref{eq:velocity_constraint} \implies \eqref{eq:increasing_distance_condition}$ by contraposition, i.e.,
	\begin{align}
		&\exists t \in [t_a, t_b]: \exists \pb \in \Orobia(t): \pb \in \mathbb{H} \lor \frac{d}{d t} \sd(\pb, \mathbb{H}) < 0 \\
		& \implies \Orobia([t_a, t_b]) \cap \Hb(\Nb, \db) \neq \varnothing \, \lor \, \exists h=1, \dots, H: \nonumber \\
		& \qquad \qquad \qquad (\exists \pb \in \Orobia([t_a, t_b]): \Nb[h] \pb > \db[h]) \, \land \nonumber\\
		& \qquad \qquad \qquad (\exists \dot{\pb} \in \Vrobi{t_a}{t_b} : \Nb[h] \dot{\pb} < 0) \, . \nonumber
	\end{align}
    For the case that $\exists t \in [t_a, t_b]: \exists \pb \in \Orobia(t): \pb \in \mathbb{H}$, we get by definition of the reachable occupancy that $\exists t \in [t_a, t_b]: \exists \pb \in \Orobia(t): \pb \in \mathbb{H} \implies \Orobia([t_a, t_b]) \cap \Hb(\Nb, \db) \neq \varnothing$.
	For the opposite case, let $\nb_{\pb}$ be the normal vector defining the supporting hyperplane of a point $\pb \notin \mathbb{H}$ on $\mathbb{H}$, and let $\mathcal{H}_{\pb} = \left\{h \smid h=1, \dots, H, \Nb[h] \pb > \db[h]\right\}$ be the set of active polytope constraints for the point $\pb$.
	Then, we can write $\nb_{\pb} = \sum_{h \in \mathcal{H}_{\pb}} \lambda_h \Nb[h]^\top$ with $\lambda_h \ge 0$ \cite[Sec. 8.1.2]{boyd_2009_ConvexOptimization}.
	Let the halfspace defined by the hyperplane with normal vector $\nb_{\pb}$ be $\mathcal{S}(\nb_{\pb}, d_{\pb}) = \left\{\pb_{\pb} \smid \nb^\top_{\pb} \pb_{\pb} \le d_{\pb}\right\}$.
	The signed distance function for a halfspace is $\sd(\pb, \mathcal{S}(\nb_{\pb}, d_{\pb})) = \nb^\top_{\pb} \pb - d_{\pb}$ with $\frac{d}{d t} \sd(\pb, \mathcal{S}(\nb_{\pb}, d)) = \nb^\top_{\pb} \dot{\pb}$ if we assume a static polytope.
	By definition of the supporting hyperplane $\sd(\pb, \mathbb{H}) = \sd(\pb, \mathcal{S}(\nb_{\pb}, d_{\pb}))$.
	Therefore, ${\frac{d}{d t} \sd(\pb, \mathbb{H}) < 0} %
	\implies {\sum_{h \in \mathcal{H}_{\pb}} \lambda_h \Nb[h] \dot{\pb} < 0} \implies {\exists h \in \mathcal{H}_{\pb}: \Nb[h] \dot{\pb} < 0}$, which proofs \cref{prop:normal_velocity} as all possible robot velocities are contained in $\Vrobi{t_a}{t_b}$.
\end{proof}
In the following paragraph, we present how we efficiently evaluate~\eqref{eq:velocity_constraint} in practice.

\paragraph{Computation of the Reachable Velocity}
In \shield{}, we use capsules to represent the reachable occupancies of the robot links as described in~\cref{sec:safety_shield}.
We denote the linear and angular velocity of a point $\pb$ as $\vb = \left[\dot{\pb}, \bom\right]^\top$.
\begin{proposition}\label{prop:point_moving_away}
	All points in a capsule {$\mathcal{C}(\pb_1, \pb_2, r)$} with {$\vb_1 = \left[\dot{\pb}_1, \bom\right]^\top$} are moving away from a plane with normal vector $\nb$ if 
	\begin{align}
		\check{v}_{\nb} = \nb^\top \dot{\pb}_1 - (r + \, \norm{\pb_2 - \pb_1}) \, \norm{\nb \times \bom} > 0 \,.
	\end{align}
\end{proposition}

\begin{proof}
Given a vector $\nb$, we want to determine the minimal speed any point in a capsule $\mathcal{C}$ has in the direction of $\nb$ as
\begin{align}
   \check{v}_{\nb} = \min_{\pb \in \mathcal{C}} \  & \, \nb^\top \dot{\pb} \, .
\end{align}
A point {$\pb \in \mathcal{C}(\pb_1, \pb_2, r)$} can be described as {$\pb = \pb_1 + \lambda (\pb_2 - \pb_1) + \rb$} with {$\lambda \in [0, 1]$}, and {$\norm{\rb} \le r$}.
The speed of such a point in the direction of $\nb$ is 
\begin{subequations}
\begin{align}
    \nb^\top \dot{\pb} &= \nb^\top \dot{\pb}_1 + \nb^\top \left(\bom \times (\pb - \pb_1)\right) \label{eq:approximation_linear_velocity_1}\\
    &= \nb^\top \dot{\pb}_1 + (\pb - \pb_1)^\top \left(\nb \times \bom\right) \\
    &= \nb^\top \dot{\pb}_1 + \norm{\pb - \pb_1} \norm{\nb \times \bom} \cos(\alpha) \\
    &\ge \nb^\top \dot{\pb}_1 - (r + \norm{\pb_2 - \pb_1}) \norm{\nb \times \bom} = \check{v}_{\nb} \, , \label{eq:approximation_linear_velocity}
\end{align}
\end{subequations}
where $\alpha$ is the angle between $\pb - \pb_1$ and $\nb \times \bom$.
Therefore, all points on a capsule {$\mathcal{C}(\pb_1, \pb_2, r)$} move away from a plane with normal vector $\nb$ if $\check{v}_{\nb} > 0$. 
\end{proof}

We can only compute the velocity and acceleration of a point on the $i$-th robot link $\robin{\pb}{1}$ at discrete time points using the link Jacobian $\robisimp{\Jb}$ {\cite[Eq. 7.16]{siciliano_2009_RoboticsModelling}}
as $\robin{\vb}{1} = \robisimp{\Jb} \dot{\qb}$ and $\robin{\dot{\vb}}{1} = \robisimp{\dot{\Jb}} \dot{\qb} + \robisimp{\Jb} \ddot{\qb}$. 
To bound the velocity error for an entire time interval, we apply~\cref{theo:taylor_velocity_bound}.

Let $\pb$ be a point on robot link $r_i$ in the time interval $[t_a, t_b]$, and let $\bar{t} = \frac{t_a + t_b}{2}$. 
The projection of the linear velocity $\dot{\pb}$ onto a vector $\nb$ can be approximated by a first-order Taylor polynomial:
\begin{align}\label{eq:definition_taylor_polynomial}
	v^{\pb}_{\nb}(t) &= v^{\pb}_{\nb}(\bar{t}) + \left(\dot{v}^{\pb}_{\nb}(\bar{t})\right) (t-\bar{t}) + r_{\nb}(t, \bar{t}) \, ,
\end{align}
where the Lagrange remainder $r_{\nb}$ satisfies
\begin{align}
	|r_{\nb}(t, \bar{t})| \le \underset{\tau \in [\bar{t}, t]}{\sup} \frac{1}{2} |\ddot{v}^{\pb}_{\nb}(\tau)| (t-\bar{t})^2 \, .
\end{align}

\begin{theorem}\label{theo:taylor_velocity_bound}
Using the approximation in~\eqref{eq:definition_taylor_polynomial}, the speed of any point on the $i$-th robot link $\pb \in \mathcal{C}(\robin{\pb}{1}, \robin{\pb}{2}, \robi{r}) \coloneqq \robisimp{\mathcal{C}}$ projected onto the vector $\nb$ with $\norm{\nb}=1$ in the time interval $[t_a, t_b]$ is bounded below by:
\begin{align}\label{eq:full_vel_bound}
	\check{v}_{\nb, i} &= \underset{\pb \in \robisimp{\mathcal{C}}, t \in [t_a, t_b]}{\min} v^{\pb}_{\nb}(t) \\
    &\ge \nb^\top \robin{\dot{\pb}}{1} - w_{v, \, \text{max}} - \robisimp{l} \left(\norm{\nb \times \robi{\bom}} + w_{\omega, \, \text{max}}\right) \label{eq:taylor_approx} \\
    & \quad \ + \dfrac{\Delta t}{2} \bigg( \nb^\top \robin{\ddot{\pb}}{1} - w_{a, \, \text{max}} - \robisimp{l} \big(\norm{\nb \times \robi{\dot{\bom}}} \, + \nonumber \\ 
    &\qquad \qquad \quad \norm{\robi{\bom}} \norm{\nb \times \robi{\bom}} + w_{\dot{\omega}, \, \text{max}} \, + \nonumber \\
    &\qquad \qquad w_{\omega, \, \text{max}} (\norm{\robi{\bom}} + \norm{\nb \times \robi{\bom}} + w_{\omega, \, \text{max}} )\big) \bigg) \nonumber \\
    &\quad \ - \dfrac{\Delta t^2}{8} \sum_{j=1}^{i} l_j \left(\ddommax{j} + 3 \, \dommax{j} \, \ommax{j} + \dommax{j}^2\right) \, , \nonumber
\end{align}
where $\Delta t = t_b - t_a$, $w_{v, \, \text{max}}$, $w_{\omega, \, \text{max}}$, $w_{a, \, \text{max}}$, and $w_{\dot{\omega}, \, \text{max}}$ are the maximal absolute estimation errors of the linear and angular velocity and acceleration, and $\robisimp{l} = \norm{\robin{\pb}{2} - \robin{\pb}{2}} + \robi{r}$, and 
\begin{subequations}
\begin{align}
    \ommax{j} &\coloneqq \sum_{k=1}^{j} \dqmax{k} \, ,  \\
    \dommax{j} &\coloneqq \sum_{k=1}^{j} \ddqmax{k} + \dqmax{k} \, \ommax{k-1} \, , \\
    \ddommax{j} &\coloneqq \sum_{k=1}^{j} \dddqmax{k} + 2 \ddqmax{k} \ommax{k-1} + \dqmax{k} \left(\dommax{k-1} + \ommax{k-1}^2\right) \, , \\
    l_j &\coloneqq 
    \begin{cases}
        \norm{\robjn{\pb}{2}-\robjn{\pb}{1}} \, , & j < i\\
        \norm{\robjn{\pb}{2}-\robjn{\pb}{1}} + r^r_j \, , & j=i \, .
    \end{cases}\label{eq:length_definition}
\end{align}
\end{subequations}
\end{theorem}
\begin{proof}
	See Appendix~\cref{sec:app1_vel_err}.
\end{proof}

Using $\check{v}_{\nb, i}$, we incorporate the ECC constraint in \eqref{eq:ecc_single_2} into a new constraint to exclude cases, where a robot link $r_i$ is moving away from the surface of an environment element $e_k$, using 
\begin{align}\label{eq:ecc_single_3}
    \bar{c}_{\text{ECC}, i, j, k}''\left([t_a, t_b]\right) &\coloneqq \bar{c}_{\text{ECC}, i, j, k}'\left([t_a, t_b]\right) \land \bar{c}_{\text{contact}, i, k}\left([t_a, t_b]\right) \nonumber\\
    & \qquad \qquad \bigwedge_{\nb \in \mathcal{N}_{i, k}} \left(\check{v}_{\nb, i} \ge 0\right) \, . 
\end{align}
Correspondingly, we define a new constraint based on the SCC constraint between two robot links $r_i$ and $r_l$ in~\eqref{eq:scc_single_2} as
\begin{align}\label{eq:scc_single_3}
    \bar{c}_{\text{SCC}, i, j, l}''\left([t_a, t_b]\right) &\coloneqq \bar{c}_{\text{SCC}, i, j, l}'\left([t_a, t_b]\right) \land  \bar{c}_{\text{contact}, i, l}\left([t_a, t_b]\right) \nonumber\\
    & \qquad \qquad \bigwedge_{\nb \in \mathcal{N}_{i, l}} \left(\check{v}_{\nb, i} \ge \check{v}_{\nb, l}\right) \, . 
\end{align}

\subsubsection{Robot Topology}\label{sec:robot_topology}
Most collaborative robots are designed to prevent clamping between certain robot links, either through their topology or by limiting the joint angles.
Hence, for a given robot, we define the set of robot link pairs that cannot cause an SCC as 
\begin{align}
 \rob{\mathcal{G}} = \left\{(r_i, r_l) \smid \text{no SCC between link~$r_i$ and $r_l$ possible} \right\}\,
\end{align}
and extend~\eqref{eq:scc_single_3} by
\begin{align}\label{eq:scc_single_4}
    &\bar{c}_{\text{SCC}, i, j, l}''' \left([t_a, t_b]\right) \coloneqq \bar{c}_{\text{SCC}, i, j, l}'' \left([t_a, t_b]\right) \lor 
    (r_i, r_l) \in \rob{\mathcal{G}} \, .
\end{align}

\subsubsection{Full Constrained Contact Detection}
With the derived constraint relaxations, we redefine the ECC and SCC constraints in~\eqref{eq:ecc} and \eqref{eq:scc} as
\begin{align}
	\bar{c}_{\text{ECC}, i, j}([t_a, t_b]) &\coloneqq \bigwedge_{e_k \in \Es} \bar{c}_{\text{ECC}, i, j, k}''\left([t_a, t_b]\right) \label{eq:ecc_2}\\
    \bar{c}_{\text{SCC}, i, j}([t_a, t_b]) &\coloneqq \bigwedge_{r_l \in \Rs \setminus r_i} \bar{c}_{\text{SCC}, i, j, l}'''\left([t_a, t_b]\right) \, . \label{eq:scc_2}
\end{align}

\subsection{Handling Multi-Body Clamping}
\begin{figure}
    \centering
    \includegraphics[width=0.9\columnwidth, trim=0 3mm 0 0mm, clip]{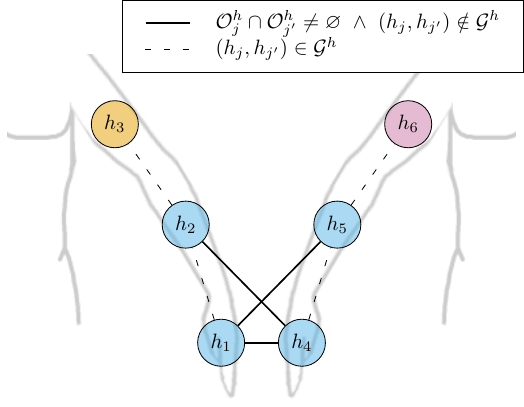}
    \caption{Example of a simplified contact graph $\mathbb{G}$ with two humans in the scene. 
    All connected components in the graph are colored differently.
    The connected component of $\Hs_{c, 1} = \left\{h_{1}, h_{2}, h_{4}, h_{5}\right\}$ is combined to a new body part and the set of all connected components is $\Hs_{c} = \left\{\Hs_{c, 1}, \left\{h_3\right\}, \left\{h_6\right\} \right\}$.}
    \label{fig:contact_graph}
\end{figure}

An important scenario to consider is the case, where the robot clamps multiple human body parts, e.g., both hands are clamped in between the robot and a table.
In these cases, the constraint in~\eqref{eq:constrained_contact_constraint} does not guarantee the detection of all possible constrained contacts anymore.
To handle the clamping of multiple human body parts at the same time, we combine all body parts that could be in contact with each other to a new virtual body part.
Fortunately, within a human body, there are a number of body part combinations that cannot lead to critical constrained contact and can thereby be excluded, e.g., clamping the right hand and the right lower arm.
Therefore, we define the set of safe human body part pairs as
\begin{align}
    &\hum{\mathcal{G}} = \{(h_j, h_{j'}) \smid  
    \text{no clamping of $h_j$ and $h_{j'}$ possible } \land \nonumber\\
    & \qquad \qquad \quad \text{ $h_j$ and $h_{j'}$ belong to the same human}\} \, . 
\end{align} 
To find all connected body parts, we construct an undirected graph $\mathbb{G}$, where the vertices are the body parts $h_j \in \Hs$ of all humans in the scene.
The edges between two vertices $h_{j}$ and $h_{j'}$ are denoted as $\langle h_j, h_{j'}\rangle$.
We construct an edge
\begin{align}
    \langle h_j, h_{j'}\rangle \ \text{if} \ \Ohumja \cap \Ohumjb \neq \varnothing \ \land \ (h_j, h_{j'}) \notin \hum{\mathcal{G}} \, .
\end{align}
\Cref{fig:contact_graph} provides an example of a simplified graph with two human arms in the scene.

We use a depth-first search algorithm~\cite{hopcroft_1973_Algorithm447} to find the set of all connected components\footnote{A connected component is a set of vertices where each vertex is reachable from any other vertex in the set, and no additional vertices can be added while preserving this property.} $\Hs_c \subset \mathcal{P}(\Hs)$ in the graph.
We construct an additional set of human body parts $\tilde{\Hs}$ by creating a new combined body part per connected component $\Hs_{c, a} \in \Hs_c: |\Hs_{c, a}| > 1$.
A new combined body part $h_{\tilde{j}} \in \tilde{\Hs}$ is built from the elements in $\Hs_{c, a} \in \Hs_c$ and has a diameter, occupancy, and admissible contact energy of
\begin{subequations}
	\begin{align}
		\humjc{d} &\coloneqq \sum_{h_j \in \Hs_{c, a}} \humj{d} \label{eq:combined_body_parts_1}\\
		\humjc{\mathcal{O}} &\coloneqq \bigcup_{h_j \in \Hs_{c, a}} \Ohumja  \label{eq:combined_body_parts_2}\\
		T_{\text{clamp}, i, \tilde{j}} &\coloneqq \underset{h_j \in \Hs_{c, a}}{\min} T_{\text{clamp}, i, j} \, . \label{eq:combined_body_parts_3}
	\end{align}
\end{subequations}

We guarantee the safety for these combined body parts by rewriting the constraints in \eqref{eq:combined_constraint} and \eqref{eq:complete_constraint} to the new constraint
\begin{align}\label{eq:combined_body_constraint}
	c'_{\text{safe}}(\chix{M})  &\coloneqq c_{\text{safe}}(\chix{M}) \bigwedge_{a=0}^{M-1} \bigwedge_{r_i \in \Rs} \bigwedge_{h_{\tilde{j}} \in \tilde{\Hs}} 
    \bar{c}_{\text{contact}, i, \tilde{j}}([t_a, t_{a+1}]) \, \lor \nonumber\\ 
    & \ \ \ c_{\text{free}, i, \tilde{j}}([t_a, t_{a+1}]) \lor c_{T, \, \text{clamp}, i, \tilde{j}}([t_a, t_{a+1}]) \, .
\end{align}
Note, that we do not verify the unconstrained contact energy for the combined body parts in~\eqref{eq:combined_body_constraint} as these are already verified for the individual bodies in $c_{\text{safe}}(\chix{M})$.
In Appendix~\cref{sec:app2_combined_body_parts} we show that if \eqref{eq:combined_body_constraint} holds for a combined body part, it also holds for all constrained contacts involving all possible combinations of body parts in that combined body part.
\section{Contact energy constraints}\label{sec:energy_constraints}
\begin{table*}[t]
	\centering
	\caption{Admissible contact energies.}
	\label{tab:contact_energies}
	\begin{tabular}{l cccc cccc}
		\toprule
		\multirow{2}{*}{\diagbox{Body Part}{Robot Geometry}} & \multicolumn{4}{c}{Constrained Contact $T_{\text{clamp}, i, j}$ [\si{\joule}]} & \multicolumn{4}{c}{Unconstrained Contact $T_{\text{free}, i, j}$ [\si{\joule}]}\\
		\cmidrule(lr){2-5} \cmidrule(lr){6-9}
         & Blunt$^\dagger$ & Wedge$^\star$ & Edge$^\star$ & Sheet$^\star$ & Blunt$^\dagger$ & Wedge$^\ddagger$ & Edge$^\ddagger$ & Sheet$^\ddagger$\\
		\midrule
		Hand       & 0.49  & 0.05 & 0.02 & 0.11 & 0.49 & 2.0  & 0.375 & 0.9\\
		Lower Arm  & 1.3 & 0.05 & 0.02 & 0.11 & 1.3  & 2.0  & 0.375 & 0.9\\
		Upper Arm  & 1.5 & 0.05 & 0.02 & 0.11 & 1.5  & 0.5  & 0.2 & 0.5\\
		Torso (Chest) & 1.6 & 0.05 & 0.02 & 0.11 & 1.6  & 0.5  & 0.2 & 0.5\\
		Head (Face) & 0.11 & 0.05 & 0.02 & 0.11 & 0.11 & 0.11$^\dagger$ & 0.11$^\dagger$ & 0.11$^\dagger$\\
		\bottomrule
		\multicolumn{9}{l}{~}\\[-0.22cm]
		\multicolumn{9}{l}{
			\footnotesize{
				$^\dagger$From~\cite{iso_2024_RoboticsSafety},
				$^\star$From~\cite{kirschner_2024_SafeRobot}, $^\ddagger$From~\cite{kirschner_2024_UnconstrainedCollision}
			}
		}
	\end{tabular}
\end{table*}

This section discusses how we evaluate the contact energy constraints $c_{T, \, \text{free}, i, j}$ and $c_{T, \, \text{clamp}, i, j}$ in~\eqref{eq:combined_body_constraint}.
For this, we derive the effective kinetic energy of the robot with respect to a contact with robot link~$r_i$ in a given configuration and define the necessary constraints for all possible contacts.

The kinetic energy of the entire robot is~\cite[Eq. 7.31]{siciliano_2009_RoboticsModelling} 
\begin{align}
    \rob{T}(\qb) &= \frac{1}{2} \dot{\qb}^\top \Bb(\qb) \dot{\qb} \, ,
\end{align}
where $\Bb(\qb)$ is the inertia matrix.
In this work, we assume that the robot consists of rigid body parts without joint elasticity, where the motor masses and inertia are included in the links, so we can formulate the inertia matrix as~\cite[Eq. 7.32]{siciliano_2009_RoboticsModelling}
\begin{align}
\Bb(\qb)& =\sum_{i=1}^{N} \robisimp{m} \robisimp{\Jb}^{P} \robisimp{\Jb}^{P} + \robisimp{\Jb}^{O \top} \robisimp{\Rb} \robisimp{\Ib} \robisimp{\Rb}^\top \robisimp{\Jb}^{O}\, ,
\end{align}
where $\robisimp{m}$ is the mass of link~$r_i$, $\robisimp{\Jb}^{P}$ and $\robisimp{\Jb}^{O}$ are the $i$-th link Jacobians~\cite[Eq. 7.16]{siciliano_2009_RoboticsModelling} with respect to the position and orientation, $\robisimp{\Rb}$ is the rotation matrix from the $i$-th link frame to the base frame, and $\robisimp{\Ib}$ is the inertia tensor of link~$r_i$.

We show that the effective kinetic energy of the robot with respect to a point on link $i < N$ is independent of the rotational velocity of the joints $i+1, \dots, N$.
Let $\robi{\pb}$ be a point on link $r_i$.
The motion of $\robi{\pb}$ only depends on the joint velocities $q_1, \dots, q_i$, i.e., %
\begin{align} \label{eq:vel_point_link}
    \robi{\vb} = \robisimp{\Jb} \robisimp{\Eb} \dot{\qb} \, \,
\end{align}
where $\robisimp{\Jb}$ is the Jacobian with respect to the point $\robi{\pb}$ and $\robisimp{\Eb}$ is the truncated identity matrix with $1$s on the diagonal up to row~$i$ and $0$s elsewhere.
The effective kinetic energy of the robot with respect to a point $\robi{\pb}$ is given in \cite[Sec. 2]{khatib_1995_InertialProperties} as
\begin{align}\label{eq:kinetic_energy_link}
    \robi{T}(\qb) = \frac{1}{2} \robitop{\vb} \robisimp{\Lambda}(\qb) \robi{\vb} \, ,
\end{align}
where 
\begin{align}\label{eq:kin_energy_matrix}
    \robisimp{\Lambda}(\qb) = (\robisimp{\Jb}^\top)^{-1} \Bb(\qb) (\robisimp{\Jb})^{-1}
\end{align}
is the operational space kinetic energy matrix from~\cite[Eq. 2]{khatib_1995_InertialProperties}.
Inserting \eqref{eq:vel_point_link} and \eqref{eq:kin_energy_matrix} into \eqref{eq:kinetic_energy_link} directly leads to the effective kinetic energy exerted upon contact with a point on the robot link~$r_i$ as
\begin{align}
    \robi{T}(\qb) &= \frac{1}{2} \dot{\qb}^\top \robisimp{\Eb} \Bb(\qb) \robisimp{\Eb} \dot{\qb} \, .
\end{align}

Thus, we define the contact energy constraints in~\eqref{eq:complete_constraint} as
\begin{align}
	c_{T, \, \text{free}, i, j}([t_a, t_b]) &= \max \left(\robi{T}(\qb(t_a)), \robi{T}(\qb(t_b))\right) < T_{\text{free}, i, j} \label{eq:free_energy_constraint}\\
	c_{T, \, \text{clamp}, i, j}([t_a, t_b]) &= \max \left(\robi{T}(\qb(t_a)), \robi{T}(\qb(t_b))\right) < T_{\text{clamp}, i, j},\label{eq:clamp_energy_constraint}
\end{align}
where $T_{\text{\{type\}}, i, j}$ gives the maximal contact energy for the respective human body part, contact type, and link index.
We list all contact energies used in \shield{} in~\cref{tab:contact_energies}.
For the torso and head, we use the chest and face thresholds of~\iso{} as they are more conservative than the abdomen and skull values.
The link index is included in this threshold so that we can differentiate between different worst-case contact shapes on the different links.
For most collaborative robots the first $N-1$ links are designed to only induce blunt contacts, whereas the worst-case impact shape on the end effector depends on the gripper and tool.
For the limits in~\cref{tab:contact_energies}, we assume that the human skin is similar to dew claws, as discussed in~\cite[Sec. III-A]{kirschner_2024_UnconstrainedCollision}.
The constrained contact energies for the head are taken from~\cite[Tab. M.4]{iso_2024_RoboticsSafety}, which already proposes conservative values.

\section{Experiments}\label{sec:experiments}
This section tests our two main hypotheses:
\begin{itemize}
    \item[\textbf{H1}] \vspace{0em} \shield{} significantly improves the robot performance in comparison to previous safe HRC approaches.
    \item[\textbf{H2}] \vspace{0em} The kinetic energy of a robot controlled with \shield{} during a contact with a human is below injury thresholds.
\end{itemize}
We evaluate \textbf{H1} in \cref{sec:robot_efficiency} and \textbf{H2} in \cref{sec:contact_eval}.
All experiments were conducted on a Schunk LWA 4P robot. 

\subsection{Robot Efficiency Evaluation}\label{sec:robot_efficiency}
In this subsection, we evaluate \textbf{H1} in simulation using recorded human motions in~\cref{sec:robot_efficiency_sim} and on a real HRC setup in~\cref{sec:robot_efficiency_real}.

\subsubsection{Simulated Robot Efficiency Study}\label{sec:robot_efficiency_sim}
\newcommand{\imgsuffix}{}
\ifthenelse{\boolean{anon}}{
	\renewcommand{\imgsuffix}{_anon}
}{}
\edef\stackingimg{stacking\imgsuffix.png}
\edef\towerimg{tower\imgsuffix.png}
\edef\puzzleimg{puzzle\imgsuffix.png}
\edef\screwingimg{screwing\imgsuffix.png}
\edef\robcoimg{robco\imgsuffix.png}
\begin{figure}[t]
    \subfloat[Stacking\label{fig:tasks:stacking}]{
        \includegraphics[width=0.32\columnwidth]{\stackingimg}
    }
    \hspace{-0.3cm}
    \subfloat[Tower\label{fig:tasks:tower}]{
        \includegraphics[width=0.32\columnwidth]{\towerimg}
    }
    \hspace{-0.3cm}
    \subfloat[Puzzle\label{fig:tasks:puzzle}]{
        \includegraphics[trim=7cm 6cm 5cm 3cm, clip, width=0.32\columnwidth]{\puzzleimg}
    }
    
    \vspace{-0.2em} %
    \subfloat[Screwing\label{fig:tasks:screwing}]{
        \includegraphics[trim=1cm 2cm 7cm 4cm, clip, width=0.48\columnwidth]{\screwingimg}
    }
    \hspace{-0.3cm}
    \subfloat[Assembly\label{fig:tasks:robco}]{
        \includegraphics[trim=5cm 5.25cm 6cm 3cm, clip, width=0.48\columnwidth]{\robcoimg}
    }

    \caption{The five HRC tasks evaluated in our ablation study.}
    \label{fig:tasks}
\end{figure}

\begin{table*}[t]
    \centering
    \caption{Mean and standard deviation of the robot efficiency across \num{30} trials.}
    \begin{tabular}{lcccccccc}
\toprule
  & \multicolumn{2}{c}{Efficiency per End Effector Type [\%]} &    &     &      \\
\cmidrule(lr){2-3}
Method & Blunt & Sharp & p-value of \textbf{{H1}} & Computation Time [ms] & Provably safe w.r.t.~\cref{sec:assumptions} \\
\midrule
SSM zone \cite{iso_2021_RoboticsSafety} & 6.8 $\pm$ 10.7 & 6.8 $\pm$ 10.7 & 8.14e-12 & \textbf{0.09} $\pm$ 0.01 & \ding{51} \\
Reduced-speed PFL \cite{iso_2021_RoboticsSafety} & 38.5 $\pm$ 8.5 & 38.5 $\pm$ 8.5 & 8.14e-12 & 0.12 $\pm$ 0.02 & \ding{53} \\
Dynamic SSM \cite{beckert_2017_OnlineVerification, thumm_2022_ProvablySafe} & 59.2 $\pm$ 19.7 & 59.2 $\pm$ 19.7 & 8.14e-12 & 0.15 $\pm$ 0.04 & \ding{51} \\
Reduced-speed zone \cite{svarny_2019_SafePhysical} & 64.5 $\pm$ 13.8 & 64.5 $\pm$ 13.8 & 8.14e-12 & 0.22 $\pm$ 0.04 & \ding{53} \\
Reflected mass \cite{haddadin_2012_MakingRobots} & 83.4 $\pm$ 3.4 & 80.4 $\pm$ 4.7 & 1.10e-11 & 1.00 $\pm$ 5.02 & \ding{53} \\
\shield{} without $c_{\mathrm{free}}$ & 82.3 $\pm$ 13.1 & 82.3 $\pm$ 7.8 & 8.15e-12 & 0.36 $\pm$ 0.05 & \ding{51} \\
\shield{} (\textbf{ours}) & \textbf{92.6} $\pm$ 6.3 & \textbf{92.5} $\pm$ 6.2 & -- & 0.41 $\pm$ 0.29 & \ding{51} \\
\bottomrule
\end{tabular}
    \label{tab:ablation}
\end{table*}

To fairly compare the robot efficiency across different methods, we gathered a dataset of human motion capture data in our lab while the robot was stopped\footnote{This study is exempt from ethics review because it involved non-invasive motion capture of adult participants performing standard movements in a laboratory setting. No personal or identifiable information was collected, participation was voluntary and without compensation, and no interventions or sensitive questions were involved.}.
For this, we defined the five tasks depicted in~\cref{fig:tasks}, which require close collaboration between the human and the robot and should emulate typical real-world tasks:
\begin{itemize}
    \item \textbf{Stacking}: the user has to stack two \num{10}$\times$\num{10}$\times$\SI{10}{\centi\metre} blocks in multiple pre-defined positions.
    \item \textbf{Tower}: the user has to stack eight \num{10}$\times$\num{10}$\times$\SI{10}{\centi\metre} blocks to a tower at two pre-defined positions.
    \item \textbf{Puzzle}: the user has to solve a simple puzzle.
    \item \textbf{Screwing}: the user has to screw ten screws into random positions using an allen wrench.
    \item \textbf{Assembly}: the user has to assemble two modules of a modular robot of the company RobCo.
\end{itemize}
All tasks were performed for three minutes in the robot workspace, and the task was repeated if finished early.
We collected data from six participants with a randomized order of tasks, leading to a total number of \num{30} trials.

For all tasks, the robot trajectory emulates a typical pick and place motion, where the robot approaches two pre-defined positions on the table located on the outer edges of the combined human and robot workspace.
This trajectory contains both potential high-speed unconstrained contacts and constrained contacts with the table.
Our supplementary video further visualizes these five tasks and the robot trajectory.

We simulate the robot behavior given the human measurements from the motion capture recordings across eight different methods\ifthenelse{\boolean{anon}}{:}{, all of which are available in our open-source code release:}
\begin{itemize}
    \item \textbf{No shield}: the robot executes its trajectory at full speed.
    \item \textbf{SSM zone}: if any body part is closer than \SI{1.17}{\meter} to the robot base, we stop the robot. This represents the separation distance speed and separation monitoring approach in \iso{}, and the threshold is calculated according to~\cite[Eq. 7]{svarny_2019_SafePhysical}.
    \item \textbf{Reduced-speed PFL}: the maximal Cartesian speed of the robot is limited to \SI{0.25}{\meter \per \second} representing the reduced-speed power and force limiting approach in \iso{}\footnote{\label{foot:not_safe}This approach does not guarantee safety for all robot masses and sharp robot geometries.}.
    \item \textbf{Dynamic SSM}: we use the set-based reachability speed and separation monitoring method of~\cite{beckert_2017_OnlineVerification, thumm_2022_ProvablySafe} that stops the robot at the latest possible moment.
    \item \textbf{Reduce speed zone}: we follow the approach of~\cite{svarny_2019_SafePhysical} and let the robot operate at full speed if the human is far away and switch to reduced-speed PFL if any body part is closer than \SI{0.73}{\meter} to the robot base\footnoteref{foot:not_safe}.
    \item \textbf{Reflected mass}: the reflected mass method proposed by \citet{haddadin_2012_MakingRobots}, where we adapt their safe motion unit velocities to match our energy limits in~\cref{tab:contact_energies}, ensuring a fair comparison. Since this approach has no contact type classification, we are using the more conservative constrained contact thresholds\footnote{This approach assumes a static human for the reflected mass calculation.}.
    \item \textbf{\shield{} without} $c_{\text{free}}$: our approach without contact type classification, where we use the more conservative constrained contact thresholds for all possible contacts.
    \item \textbf{\shield{}}: our proposed method with contact type classification and contact energy reduction.
\end{itemize}
In~\cref{tab:ablation}, we report the mean and standard deviation of the percentage of the total trajectory covered by each method in relation to the execution of the \textbf{no shield} method.
This metric can be interpreted as the task completion percentage of the collaborative robot and directly evaluates the robot efficiency.
We report the results for the blunt and sharp end effector geometries, where sharp contains edge, sheet, and wedge shapes to equal parts.

The results clearly show that \shield{} outperforms all baseline methods in terms of robot efficiency.
To make statistically significant claims, we compute the robot efficiency differences per trial and end effector type of all methods to \shield{}.
Our null hypothesis is that the efficiency values of the baseline method and the ones of \shield{} come from the same distribution, and the alternative hypothesis is \textbf{H1}.
For all methods, a Shapiro-Wilk test for normality~\cite{shapiro_1965_AnalysisVariance} indicates that the efficiency difference data is not normally distributed across the trials.
Hence, we perform a Wilcoxon signed-rank test~\cite{wilcoxon_1945_IndividualComparisons} to compute the statistical significance of \textbf{H1} across all baseline methods.
The p-values in~\cref{tab:ablation} show a very strong statistical evidence for \textbf{H1}.
In fact, for all but one trial in the comparison with the reflected mass method, \shield{} performed strictly better than the baseline methods.

Our evaluation further shows that our proposed contact type classification has a strong impact on the overall robot performance.
Without this method, our approach performs similar to that of~\citet{haddadin_2012_MakingRobots} with the major difference that we do not assume a static human for our contact energy thresholds.
We further see that the end effector type does not have a strong impact on the robot efficiency for \shield{} as most contact situations are unconstrained, where all end effector types permit a relatively high robot energy. 

\cref{tab:ablation} further reports the average runtime of the methods in our custom implementation\footnote{Executed single-thread on an Intel~\textregistered{} Core~\textcopyright{} i7-10750H @ 2.60GHz.}.
Note, that we focused our implementation of the state-of-the-art methods on a fair comparability of the robot efficiency rather than runtime efficiency. 
The computational overhead of \shield{} is not drastically higher than related methods and all methods are real-time deployable. %

We further provide an assessment whether the methods provide provable human safety with respect to (w.r.t.) our assumption in~\cref{sec:assumptions}.
The reduced-speed PFL approach in~\cite{iso_2021_RoboticsSafety} and the reduced-speed zone approach in~\cite{svarny_2019_SafePhysical} are not provably safe for sharp robot geometries as even light-weight robots with sharp contacts can inflict injuries at \SI{0.25}{\meter \per \second}~\cite{kirschner_2024_SafeRobot}.
Since \citet{haddadin_2012_MakingRobots} assume a stationary human for their verification, their reflected mass approach violates our assumption that humans can move with up to \SI{1.6}{\meter \per \second} in any direction.

\begin{figure*}[t]
\hspace{1.4cm}
\colorbox{white}{\makebox[0.142\textwidth][c]{\textbf{Hand}}}\hspace{0.65cm}%
\colorbox{white}{\makebox[0.142\textwidth][c]{\textbf{Lower arm}}}\hspace{0.65cm}%
\colorbox{white}{\makebox[0.142\textwidth][c]{\textbf{Upper arm}}}\hspace{0.65cm}%
\colorbox{white}{\makebox[0.142\textwidth][c]{\textbf{Chest}}}\hspace{0.65cm}%
\colorbox{white}{\makebox[0.142\textwidth][c]{\textbf{Face}}}\\[-0.0cm]

\raisebox{5\normalbaselineskip}[0pt][0pt]{\makebox[0pt]{\rotatebox[origin=c]{90}{\textbf{Edge \quad}}}}
\hspace{0.5cm}
\subfloat{\includegraphics[height=0.19\textwidth]{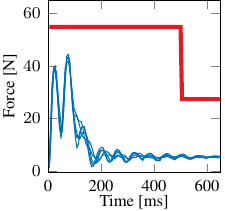}}\hspace{0.10cm}
\subfloat{\includegraphics[height=0.19\textwidth]{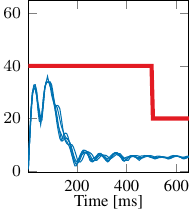}}\hspace{0.10cm}
\subfloat{\includegraphics[height=0.19\textwidth]{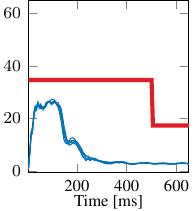}}\hspace{0.10cm}
\subfloat{\includegraphics[height=0.19\textwidth]{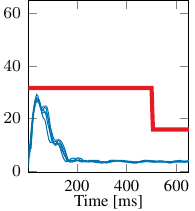}}\hspace{0.10cm}
\subfloat{\includegraphics[height=0.19\textwidth]{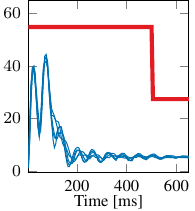}}\\[-0.2cm]

\raisebox{5\normalbaselineskip}[0pt][0pt]{\makebox[0pt]{\rotatebox[origin=c]{90}{\textbf{Wedge \quad}}}}
\hspace{0.4cm}
\subfloat{\includegraphics[height=0.19\textwidth]{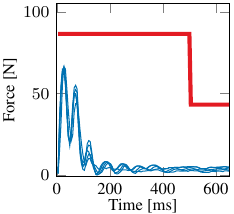}}\hspace{0.01cm}
\subfloat{\includegraphics[height=0.19\textwidth]{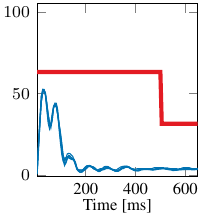}}\hspace{0.01cm}
\subfloat{\includegraphics[height=0.19\textwidth]{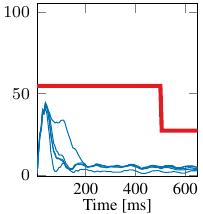}}\hspace{0.01cm}
\subfloat{\includegraphics[height=0.19\textwidth]{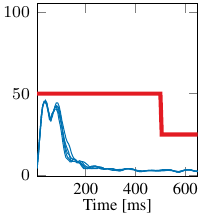}}\hspace{0.01cm}
\subfloat{\includegraphics[height=0.19\textwidth]{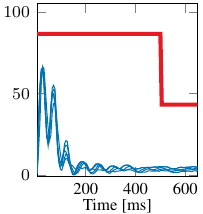}}\\[-0.2cm]

\raisebox{5\normalbaselineskip}[0pt][0pt]{\makebox[0pt]{\rotatebox[origin=c]{90}{\textbf{Sheet \quad}}}}
\hspace{0.4cm}
\subfloat{\includegraphics[height=0.19\textwidth]{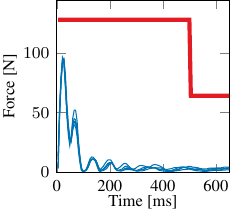}}\hspace{-0.05cm}
\subfloat{\includegraphics[height=0.19\textwidth]{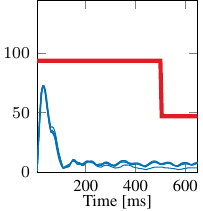}}\hspace{-0.05cm}
\subfloat{\includegraphics[height=0.19\textwidth]{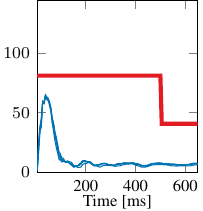}}\hspace{-0.05cm}
\subfloat{\includegraphics[height=0.19\textwidth]{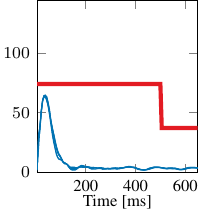}}\hspace{-0.05cm}
\subfloat{\includegraphics[height=0.19\textwidth]{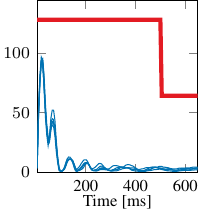}}\\[-0.2cm]

\raisebox{5\normalbaselineskip}[0pt][0pt]{\makebox[0pt]{\rotatebox[origin=c]{90}{\textbf{Blunt \quad}}}}
\hspace{0.4cm}
\subfloat{\includegraphics[height=0.19\textwidth]{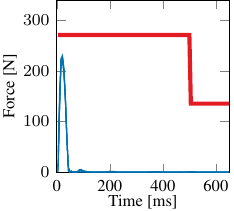}}\hspace{-0.05cm}
\subfloat{\includegraphics[height=0.19\textwidth]{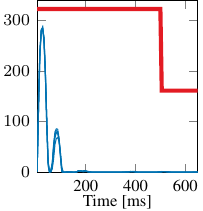}}\hspace{-0.05cm}
\subfloat{\includegraphics[height=0.19\textwidth]{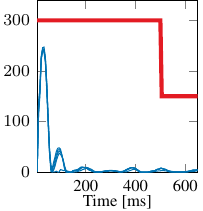}}\hspace{-0.05cm}
\subfloat{\includegraphics[height=0.19\textwidth]{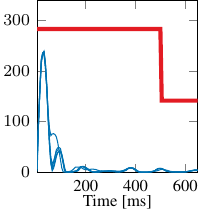}}\hspace{-0.05cm}
\subfloat{\includegraphics[height=0.19\textwidth]{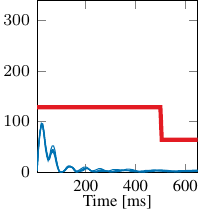}}\\[-0.5cm]
\caption{Force measurements for the Schunk robot using the Pilz robot measurement system for varying body parts (columns) and and end effector types (rows). 
The admissible force limits are depicted in red and the measured forces across five measurements in blue. 
The force limits in rows one to three are derived from the energy limits in~\cref{tab:contact_energies} using \eqref{eq:force_from_energy} and the limit in row four is from \iso{}.
}
\label{fig:pilz_force_measurements}
\end{figure*}

\subsubsection{Real-World HRC Evaluation}\label{sec:robot_efficiency_real}
To demonstrate the performance of \shield{}, we compare its real-world behavior with the speed and separation monitoring approach in~\cite{beckert_2017_OnlineVerification, thumm_2022_ProvablySafe} (dynamic SSM).
For this, we deploy both methods on a Schunk LWA 4P robot controlled by a speedgoat real-time target machine at a control frequency of \SI{166.7}{\hertz}.
The user performs the collaborative tasks described in~\cref{sec:robot_efficiency} in the workspace of the moving robot.
For this experiment, we set the end effector geometry to the edge type.

The performance difference of \shield{} and the dynamic SSM approach is best visualized in our supplementary video.
Here, we can see that the robot operated with \shield{} only has to slow down negligibly when moving from one pick location to another (unconstrained contact).
The dynamic SSM robot comes to a full stop or slows down significantly in these cases.
When a constrained contact is possible, our approach allows the robot to continue operation with a safe speed, whereas the dynamic SSM robot has to come to a full stop and wait for the user to clear the area.

To also give a quantitative analysis of this deployment, we again record the human motion and compare the executed robot path with a simulated execution without safety shield\footnote{A real-world deployment without safety shield would not be safe, so we did not gather real-world collaborative data for this setting.}.
Hereby, the robot operated with \shield{} achieves a relative performance of \num{93.7} $\pm$ \SI{4.9}{\percent} and the dynamic SSM robot falls short with a performance of \num{55.9} $\pm$ \SI{19.2}{\percent}.
These findings align with the results of~\cref{tab:ablation}, showing that the performance does not differ significantly when the user is aware of the robot.
These findings further support the validity of \textbf{H1}.

\subsection{Contact Energy Evaluation}\label{sec:contact_eval}
To evaluate \textbf{H2}, we conduct collision experiments for constrained contacts in~\cref{sec:constrained_contact_eval} and for unconstrained contacts in~\cref{sec:unconstrained_contact_eval}.

\subsubsection{Constrained Contact Validation}\label{sec:constrained_contact_eval}

\iso{} and ISO/TS 15066~\cite{iso_2016_RobotsRobotic} specify admissible contact forces given blunt robot geometries and reference stiffness values for all human body parts.
The goal of this subsection is to validate if a robot controlled with \shield{} adheres to these force limits in constrained contacts.
For this, we measure the contact forces of the robot during a constrained contact with a Pilz robot measurement system, which is certified for ISO/TS 15066~\cite{iso_2016_RobotsRobotic} and emulates varying contact stiffness.
During our tests, we manually set the human pose measurements so that the human body part under test is located at the sensor location and a constrained contact is detected.

To this date, no ISO norm specifies force limits for sharp contacts.
However, they provide a formula to convert the contact energy into a maximal transient contact force~\cite[Eq. M.1]{iso_2024_RoboticsSafety}
\begin{align}\label{eq:force_from_energy}
    F_{\text{clamp}, i, j} &= \sqrt{2 \humj{K} T_{\text{clamp}, i, j}} \, ,
\end{align}
where $\humj{K}$ is the reference stiffness of the human body part.
We use \eqref{eq:force_from_energy} to convert the contact energies in~\cref{tab:contact_energies} to admissible force limits for all end effector geometries, and perform five force measurements per end effector geometry (edge, wedge, sheet, and blunt) and human body part modeled in \shield{} (hand, lower arm, upper arm, chest, and face).

Since the primary purpose of this article is to derive a framework that effectively limits the transient contact forces, we deploy a simple contact reaction scheme using an external force sensor mounted on the end effector to limit the quasi-static forces.
Our contact reaction scheme reacts within \SI{100}{\milli\second}, which aligns with recently proposed methods using internal joint torque or current measurements~\cite{suita_1995_FailuresafetyKyozon, deluca_2006_CollisionDetection, kokkalis_2018_ApproachImplementing}.

The results in~\cref{fig:pilz_force_measurements} show that \shield{} successfully limits the transient contact forces to the given limits for all end effector geometries and human body parts.
We report that no force measurement exceeds the respective force limit, which confirms \textbf{H2} for constrained contacts.

\subsubsection{Unconstrained Contact Validation}\label{sec:unconstrained_contact_eval}

\begin{figure}[t]
	\centering
	\includegraphics[width=0.9\linewidth]{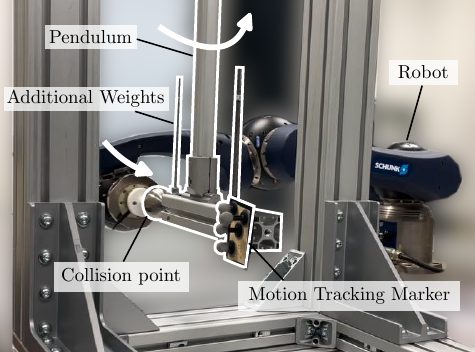}
	\caption{The pendulum experiment setup. The robot end effector collides with the pendulum, causing it to swing backwards. A Vicon motion tracking system measures the position of the pendulum. The mass of the pendulum can be increased by adding weights to the two rods. The robot stops at the point of contact.}
	\label{fig:pendulum_setup}
\end{figure}

To validate that the kinetic energy of the robot is below the unconstrained contact energy thresholds in~\cref{tab:contact_energies}, we set up an experiment where the robot collides with a pendulum at different speeds, as depicted in~\cref{fig:pendulum_setup}.
The robot moves along a trajectory such that the normal vector of the contact is perpendicular to the rotation axis of the pendulum. 
The end effector and pendulum are stiff, so we assume there is an elastic contact between them.
Upon contact, the pendulum swings backward, and the robot stops once the contact is detected by the external sensor as described in the previous subsection.

We measure the position of the pendulum with a motion-tracking system and calculate the maximum height of its center of mass in the swing-back motion using simple trigonometry.
Using the height measurement and the known pendulum mass, we can determine the potential energy of the pendulum at its peak $T_{\text{pot}} = m_{\text{pendulum}} \, g \, \Delta \mathsf{z}_{\text{COM}}$, where $g=\SI{9.81}{\meter \per \second \squared}$.
As the friction in the pendulum is negligible, we assume that the kinetic energy of the pendulum after the contact equals its peak potential energy.
Due to the law of conservation of energy, the kinetic energy of the robot at the time of contact can never exceed the potential energy of the pendulum. 

In our pendulum experiment, we define the energy thresholds $T_{\text{free}} = [\SI{0.1}{\joule}, \, \SI{0.25}{\joule}, \,\SI{0.4}{\joule}, \,\SI{0.55}{\joule}, \,\SI{0.7}{\joule}]$ of \shield{} and emulate a human body part $h_j$ positioned at the contact point.
For each energy value, we let the robot collide with a pendulum of increasing mass $m_{\text{pendulum}} = [\SI{1.132}{\kilogram}, \, \SI{2.132}{\kilogram}, \,\SI{3.132}{\kilogram}, \,\SI{4.132}{\kilogram}, \,\SI{5.132}{\kilogram}, \,\SI{6.132}{\kilogram}]$, with five contacts per mass value\footnote{\label{foot:highest_mass}We left out the mass of \SI{6.132}{\kilogram} for the highest energy of \SI{0.7}{\joule} to prevent damage to the robot hardware.}.
The calculated reflected mass~\cite{khatib_1995_InertialProperties} of the robot at the contact point is \SI{4.261}{\kilogram}, so these pendulum masses cover a range including the robot mass.
These masses also cover the typical effective masses of the human arm given in \iso{}, which are \SI{0.6}{\kilogram} for the hand, \SI{2.6}{\kilogram} for the forearm, and \SI{5.6}{\kilogram} for the entire arm.

\begin{figure}[t]
	\centering
	\includegraphics[width=0.9\linewidth]{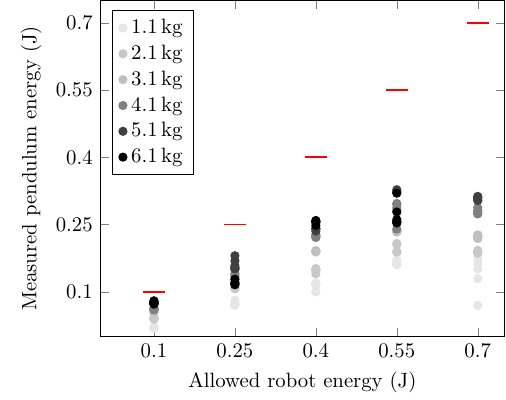}
	\caption{Results of the kinetic energy validation using a pendulum setup for varying pendulum masses. All measured pendulum energies are below the set maximal threshold of \shield{}\footnoteref{foot:highest_mass}.}
	\label{fig:pendulum_results}
\end{figure}

The results in~\cref{fig:pendulum_results} show that the measured pendulum energy does not exceed the maximal allowed kinetic energy during contact for all masses.
This supports the validity of \textbf{H2} for unconstrained contacts.
For higher velocities, the measured energy was strongly below the allowed contact energy, which might indicate that our assumption of an elastic contact is incorrect.
Nevertheless, the kinetic energy induced in the contact is lower than the respective limits, which would be safe for all human contacts.

\section{Conclusion}\label{sec:conclusion}
\shield{} is the first provably safe power and force limiting approach that ensures safe contact energies for different body parts, sharp robot geometries, and dynamic human motion.
Our approach is certifiable as it incorporates time and measurement delays for all human and robot motion and ensures safety through formal methods. 
In comparison to previous work, we do not assume a static human or simply verify against the most conservative force limit.
In extensive experimental evaluations, our approach significantly improved the robot performance over all state-of-the-art approaches.
The key methodology enabling this non-conservative behavior is classifying contacts by their type, as unconstrained contacts permit a significantly higher contact energy than constrained contacts.
With the significant improvements to robot performance in close HRC, \shield{} will accelerate the deployment of autonomous agents in critical human environments.
Future work aims at extending our framework to mobile manipulators and humanoid robots, with a particular focus on operational-space control.

\appendices

\section{Derivation of velocity error}\label{sec:app1_vel_err}
In this section, we prove~\cref{theo:taylor_velocity_bound}.
Throughout this section, we use the relations
\begin{align}
	\boldsymbol{a}^\top (\boldsymbol{b} + \boldsymbol{c}) &= \boldsymbol{a}^\top \boldsymbol{b} + \boldsymbol{b}^\top \boldsymbol{c} \label{eq:dot_distributive_relation} \\
	\boldsymbol{a} \times (\boldsymbol{b} + \boldsymbol{c}) &= \boldsymbol{a} \times \boldsymbol{b} + \boldsymbol{b} \times \boldsymbol{c} \label{eq:cross_distributive_relation} \\
	\boldsymbol{a}^\top (\boldsymbol{b} \times \boldsymbol{c}) &= \boldsymbol{b}^\top (\boldsymbol{c} \times \boldsymbol{a}) = \boldsymbol{c}^\top (\boldsymbol{a} \times \boldsymbol{b}) \label{eq:dot_cross_relation} \\
	\norm{\pb - \robin{\pb}{1}} &\le \norm{\robin{\pb}{2} - \robin{\pb}{1}} + \robi{r} \overset{\eqref{eq:length_definition}}= \robisimp{l} \label{eq:length_relation}\\
	\norm{\boldsymbol{a} \times \boldsymbol{b}} &= \norm{\boldsymbol{a}} \norm{\boldsymbol{b}} |\sin(\theta)| \le \norm{\boldsymbol{a}} \norm{\boldsymbol{b}} \label{eq:cross_product_relation} \\
	\boldsymbol{a}^\top \boldsymbol{b} &= \norm{\boldsymbol{a}} \norm{\boldsymbol{b}} \cos(\theta) \ge - \norm{\boldsymbol{a}} \norm{\boldsymbol{b}} \label{eq:dot_product_relation}
\end{align}
to derive this proof, where $\theta$ is the angle between the vectors $\boldsymbol{a}$ and $\boldsymbol{b}$.
\begin{proof}
We derive the approximation in three parts:
\begin{align}
	\check{v}_{\nb, i} &= \underset{\pb \in \robisimp{\mathcal{C}}, t \in [t_a, t_b]}{\min} v^{\pb}_{\nb}(t) \\
	\overset{\text{\eqref{eq:definition_taylor_polynomial}}}&\ge \underbrace{\underset{\pb \in \robisimp{\mathcal{C}}}{\min} \, v^{\pb}_{\nb}(\bar{t})}_{\rom{1}} + \dfrac{\Delta t}{2} 
	\underbrace{\underset{\pb \in \robisimp{\mathcal{C}}}{\min} \, \dot{v}^{\pb}_{\nb}(\bar{t})}_{\rom{2}} - \dfrac{\Delta t^2}{8} \underbrace{\underset{\pb \in \robisimp{\mathcal{C}}, t \in [t_a, t_b]}{\max} \, |\ddot{v}^{\pb}_{\nb}(t)|}_{\rom{3}} \label{eq:taylor_approx_three}
\end{align}
The relation in~\eqref{eq:taylor_approx_three} follows trivially from the subadditivity property of the minimum operator.
In the following paragraphs, we show the under-approximations for \rom{1}, \rom{2} and \rom{3}.

\paragraph{Lower bound of \rom{1}}
We under-approximate \rom{1} as
\begin{subequations}
\begin{align}
    v^{\pb}_{\nb} &= \nb^\top \dot{\pb} \\
    \overset{\text{\cite[Eq. 3.26]{siciliano_2009_RoboticsModelling}}}&= \nb^\top \big(\robin{\dot{\pb}}{1} + \wb_v + (\robi{\bom} + \wb_{\omega}) \times (\pb - \robin{\pb}{1})\big)\\
    \overset{\eqref{eq:approximation_linear_velocity}}&\ge \nb^\top \robin{\dot{\pb}}{1} - w_{v, \, \text{max}} - \robisimp{l} (\norm{\nb \times \robi{\bom}} + w_{\omega, \, \text{max}})\, .\label{eq:approximation_linear_velocity_with_error}
\end{align}
\end{subequations}

\paragraph{Lower bound of \rom{2}} 
Similarly to \eqref{eq:approximation_linear_velocity}, we express the velocity of any point on the $i$-th robot capsule $\pb \in \robisimp{\mathcal{C}}$ in relation to the first of the defining points of that capsule $\robin{\pb}{1}$ with
\begin{subequations}
\begin{align}
    \dot{v}^{\pb}_{\nb} &= \nb^\top \ddot{\pb} \\
    \overset{\text{\cite[Eq. 7.99]{siciliano_2009_RoboticsModelling}}}&= \nb^\top \big(\robin{\ddot{\pb}}{1} + \wb_a + (\robi{\dot{\bom}} + \wb_{\dot{\omega}}) \times (\pb - \robin{\pb}{1}) \\ 
    &\qquad \quad + (\robi{\bom} + \wb_{\omega}) \times \big((\robi{\bom} + \wb_{\omega}) \times (\pb - \robin{\pb}{1})\big)\big) \nonumber\\
    \overset{\eqref{eq:dot_distributive_relation} , \eqref{eq:dot_cross_relation}}&= \nb^\top (\robin{\ddot{\pb}}{1} + \wb_a) +
    (\pb - \robin{\pb}{1})^\top \big(\nb \times (\robi{\dot{\bom}} + \wb_{\dot{\omega}})\big) \nonumber \\ 
    &\qquad \ + \big((\robi{\bom} + \wb_{\omega}) \times (\pb - \robin{\pb}{1})\big)^\top \big(\nb \times (\robi{\bom} + \wb_{\omega})\big) \\
    \overset{\eqref{eq:dot_distributive_relation} , \eqref{eq:cross_distributive_relation}}&= \nb^\top \robin{\ddot{\pb}}{1} + \nb^\top \wb_a \\
    &\qquad \ + (\pb - \robin{\pb}{1})^\top \big((\nb \times \robi{\dot{\bom}}) + (\nb \times \wb_{\dot{\omega}})\big) \nonumber \\
    &\qquad \ + (\robi{\bom} \times (\pb - \robin{\pb}{1}))^\top \big((\nb \times \robi{\bom}) + (\nb \times \wb_{\omega})\big) \nonumber \\
    &\qquad \ + (\wb_{\omega} \times (\pb - \robin{\pb}{1}))^\top \big((\nb \times \robi{\bom}) + (\nb \times \wb_{\omega})\big) \nonumber \\ 
    \overset{\eqref{eq:cross_product_relation} , \eqref{eq:dot_product_relation}}&\ge \nb^\top \robin{\ddot{\pb}}{1} - w_{a, \, \text{max}} \\
    & \qquad - \norm{\pb - \robin{\pb}{1}} \big(\norm{\nb \times \robi{\dot{\bom}}} + w_{\dot{\omega}, \text{max}} \nonumber\\
    &\qquad \qquad + \norm{\robi{\bom}} (\norm{\nb \times \robi{\bom}} + w_{\omega, \, \text{max}}) \nonumber \\
    &\qquad \qquad + w_{\omega, \, \text{max}} (\norm{\nb \times \robi{\bom}} + w_{\omega, \, \text{max}})\big) \nonumber\\
	&\overset{\eqref{eq:length_relation}}= \nb^\top \robin{\ddot{\pb}}{1} - w_{a, \, \text{max}} - \robisimp{l} \big(\norm{\nb \times \robi{\dot{\bom}}} \label{eq:approximation_linear_acceleration}\\ 
	&\qquad + \norm{\robi{\bom}} \norm{\nb \times \robi{\bom}} + w_{\dot{\omega}, \, \text{max}} \nonumber \\
	&\qquad + w_{\omega, \, \text{max}} (\norm{\robi{\bom}} + \norm{\nb \times \robi{\bom}} + w_{\omega, \, \text{max}})\big)\nonumber \, ,
\end{align}
\end{subequations}
which provides the limit for \rom{2} in~\eqref{eq:taylor_approx_three}.

\paragraph{Lower bound of \rom{3}} We aim to derive a time-independent solution for the Lagrangian remainder that holds for all possible joint angles, velocities, accelerations, and jerks. %
The absolute value of the jerk of any point projected onto a normal vector $\nb$ is smaller or equal to its absolute jerk
\begin{align}
    |\ddot{v}^{\pb}_{\nb}| = |\nb^\top \dddot{\pb}| \le \norm{\dddot{\pb}} \, .
\end{align}
The acceleration of a point $\robi{\pb}$ on link $r_i$ is provided in~\cite[Eq. 7.99]{siciliano_2009_RoboticsModelling} as 
\begin{align}
    \robi{\ddot{\pb}} = \rob{\ddot{\pb}}_{i-1} + \robi{\dot{\bom}} \times \rb_{i-1, i} + \robi{\bom} \times (\robi{\bom} \times \rb_{i-1, i}) \, ,
\end{align}
where joint $i$ is located at $\rob{\pb}_{i-1}$, $\robi{\bom}$ is the angular velocity of the $i$-th link, and $\rb_{i-1, i} = \robi{\pb} - \rob{\pb}_{i-1}$.
The jerk of this point is
\begin{align}\label{eq:point_jerk}
    \robi{\dddot{\pb}} &= \rob{\dddot{\pb}}_{i-1} + \robi{\ddot{\bom}} \times \rb_{i-1, i} + \robi{\dot{\bom}} \times \dot{\rb}_{i-1, i} \\ 
    & \qquad + \robi{\dot{\bom}} \times \big(\robi{\bom} \times \rb_{i-1, i}\big) \nonumber\\
    & \qquad + \robi{\bom} \times \big(\robi{\dot{\bom}} \times \rb_{i-1, i} + \robi{\bom} \times \underbrace{\dot{\rb}_{i-1, i}}_{\overset{\text{\cite[Eq. 3.24]{siciliano_2009_RoboticsModelling}}}=\robi{\bom} \times \rb_{i-1, i}}\big) \nonumber \\
    &\overset{\eqref{eq:cross_distributive_relation}}= \rob{\dddot{\pb}}_{i-1} + \robi{\ddot{\bom}} \times \rb_{i-1, i} + 2 \, \robi{\dot{\bom}} \times (\robi{\bom} \times \rb_{i-1, i}) \\ 
    & \qquad + \robi{\bom} \times \big(\robi{\dot{\bom}} \times \rb_{i-1, i} + \robi{\bom} \times (\robi{\bom} \times \rb_{i-1, i})\big) \, . \nonumber
\end{align}
We use \eqref{eq:cross_product_relation} to derive an upper limit for the absolute jerk
\begin{align}
    \norm{\robi{\dddot{\pb}}} \le \, &\norm{\rob{\dddot{\pb}}_{i-1}} + \norm{\rb_{i-1, i}} \big(\norm{\robi{\ddot{\bom}}} + 3 \norm{\robi{\bom}} \norm{\robi{\dot{\bom}}} \nonumber \\
    &\qquad \qquad \qquad \qquad \qquad + \norm{\robi{\bom}}^2\big) \, . \label{eq:absolute_jerk_incremental}  
\end{align}
The angular velocity and acceleration of the $i$-th link are provided in~\cite[Eq. 3.25 and Eq. 7.96]{siciliano_2009_RoboticsModelling} as
\begin{align}
    \robi{\bom} &= \rob{\bom}_{i-1} + \dot{\qb}[i] \zb_{i-1} \, , \label{eq:angular_velocity} \\
    \robi{\dot{\bom}} &= \rob{\dot{\bom}}_{i-1} + \ddot{\qb}[i] \zb_{i-1} + \dot{\qb}[i] \rob{\bom}_{i-1} \times \zb_{i-1} \, , \label{eq:angular_acceleration}
\end{align}
where $\zb_{i-1}$ is the rotation axis of the $i$-th joint and $\norm{\zb_{i-1}}=1$.
The angular jerk is
\begin{align}
    \robi{\ddot{\bom}} &= \rob{\ddot{\bom}}_{i-1} + \dddot{\qb}[i] \zb_{i-1} + \ddot{\qb}[i] \dot{\zb}_{i-1} \\
    & \qquad + \big(\ddot{\qb}[i] \rob{\bom}_{i-1} + \dot{\qb}[i] \rob{\dot{\bom}}_{i-1} \big) \times \zb_{i-1} \nonumber \\
    & \qquad + \dot{\qb}[i] \rob{\bom}_{i-1} \times \underbrace{\dot{\zb}_{i-1}}_{\overset{\text{\cite[Eq. 7.96]{siciliano_2009_RoboticsModelling}}}=\rob{\bom}_{i-1} \times \zb_{i-1}} \nonumber \\
    &\overset{\eqref{eq:cross_distributive_relation}}= \rob{\ddot{\bom}}_{i-1} + \dddot{\qb}[i] \zb_{i-1} + 2 \ddot{\qb}[i] \rob{\bom}_{i-1} \times \zb_{i-1} + \label{eq:angular_jerk}\\
    & \qquad \dot{\qb}[i] \big(\rob{\dot{\bom}}_{i-1} \times \zb_{i-1} + \rob{\bom}_{i-1} \times (\rob{\bom}_{i-1} \times \zb_{i-1})\big) \, . \nonumber 
\end{align}
Using the joint velocity, acceleration, and jerk limits, we derive limits for the angular velocity, acceleration, and jerk as
\begin{align}
    \norm{\robi{\bom}} \overset{\eqref{eq:angular_velocity}}&\le \norm{\rob{\bom}_{i-1}} + |\dot{\qb}[i]| \label{eq:angular_velocity_limit}\\
    &\le \sum_{j=1}^{i} \dqmax{j} = \ommax{i} \nonumber \, ,\\
    \norm{\robi{\dot{\bom}}} \overset{\eqref{eq:angular_acceleration}}&\le \norm{\rob{\dot{\bom}}_{i-1}} + |\ddot{\qb}[i]| + |\dot{\qb}[i]| \, \norm{\rob{\bom}_{i-1}} \label{eq:angular_acceleration_limit}\\
    &\le \sum_{j=1}^{i} \ddqmax{j} + \dqmax{j} \, \ommax{j-1} = \dommax{i} \nonumber \, ,\\
    \norm{\robi{\ddot{\bom}}} \overset{\eqref{eq:angular_jerk}}&\le \norm{\rob{\ddot{\bom}}_{i-1}} + |\dddot{\qb}[i]| + 2 |\ddot{\qb}[i]| \, \norm{\rob{\bom}_{i-1}} + \label{eq:angular_jerk_limit}\\
    & \qquad |\dot{\qb}[i]| \, \big(\norm{\rob{\dot{\bom}}_{i-1}} + \norm{\rob{\bom}_{i-1}}^2 \big) \nonumber \\
    &\le \sum_{j=1}^{i} \dddqmax{j} + 2 \ddqmax{j} \ommax{j-1} + \dqmax{j} \big(\dommax{j-1} + \ommax{j-1}^2\big) = \ddommax{i} \, . \nonumber
\end{align}
Using these limits, we can bound the jerk of a point $\pb \in \robisimp{\mathcal{C}}$ on the $i$-th link to
\begin{align}
    \norm{\dddot{\pb}} \overset{\eqref{eq:absolute_jerk_incremental}, \eqref{eq:angular_velocity_limit}, \eqref{eq:angular_acceleration_limit}, \eqref{eq:angular_jerk_limit}}\le \, &\sum_{j=1}^{i} l_j \big(\ddommax{j} + 3 \, \dommax{j} \, \ommax{j} + \dommax{j}^2\big)\, ,
\end{align}
where the lengths $l_j$ are given in~\eqref{eq:length_definition}.
For all links $j < i$, the length is the distance between the two joints $ \norm{\rb_{j-1, j}} = l_j = \norm{\robjn{\pb}{2}-\robjn{\pb}{1}}$, and for the $i$-th link, the length is the maximal distance a point can have from the $i$-th joint $\norm{\rb_{i-1, i}} \le l_i = \norm{\robin{\pb}{2}-\robin{\pb}{1}} + \robi{r}$. 
\end{proof}

\section{Safety of Combined Body Parts}\label{sec:app2_combined_body_parts}
This section proves that verifying against the constraint for combined body parts guarantees safety.

\begin{proposition}\label{prop:combined_bodies}
    Verifying against the constraint in~\eqref{eq:combined_body_constraint} with the properties in~\eqref{eq:combined_body_parts_1} -- \eqref{eq:combined_body_parts_3} guarantees the safety of multiple body parts against clamping.
\end{proposition}
\begin{proof}
We show that (a) all possible contacts between the robot and any combined body part are detected, (b) potential constrained contacts are always detected, and (c) the contact energy is below the thresholds for all body parts in the combined body part. Let $h_{\tilde{j}}$ be the combined body part constructed from the elements of $\Hs_{c, a} \in \Hs_c$.

\textit{Proving (a)} 
If and only if the reachable occupancy of a robot link intersects any reachable occupancy of any body in the combined body part, it intersects with the reachable occupancy of the combined body
\begin{align}\label{eq:multi-body-proof-a}
    &\exists h_{j'} \in \Hs_{c, a}: \humjb{\mathcal{O}}([t_a, t_b]) \cap \Orobia([t_a, t_b]) \neq \varnothing \nonumber\\
    & \quad \iff \humjc{\mathcal{O}}([t_a, t_b]) \cap \Orobia([t_a, t_b]) \neq \varnothing \, , 
\end{align}
due to the relation in \eqref{eq:combined_body_parts_2}, which directly proves (a).

\textit{Proving (b)} We first show that all possible ECCs and SCCs are detected.
If we replace $\Orobia([t_a, t_b])$ in \eqref{eq:multi-body-proof-a} with $\Oenvk$ for ECCs we directly show the validity of $\bar{c}_{\text{\{type\}}, i, \tilde{j}, k}([t_a, t_{b}])$.
The same applies for ECCs, where we replace $\Orobia([t_a, t_b])$ in \eqref{eq:multi-body-proof-a} with $\Orobla([t_a, t_b])$.
Since a set of human bodies can only be clamped if they touch, verifying the constraints in~\eqref{eq:ecc_single_2} and \eqref{eq:scc_single_2} against the diameter in~\eqref{eq:combined_body_parts_1} trivially ensures safety.
The robot velocity constraint in~\eqref{eq:ecc_single_3} and the robot topology constraint in~\eqref{eq:scc_single_4} are independent of the human body part, so safety remains guaranteed.
Therefore, the extended constraints for reducing conservatism in~\eqref{eq:ecc_2} and \eqref{eq:scc_2} still hold.

\textit{Proving (c)} 
As $\robisimp{T} < \min(T_{i, j}, T_{i, j'}) \iff \robisimp{T} < T_{i, j} \land \robisimp{T} < T_{i, j'}$, the contact energy threshold in \eqref{eq:combined_body_parts_3} ensures that no contact energy threshold of any body part in $\Hs_{c, a}$ is violated if $c_{T, \, \text{clamp}, i, \tilde{j}}$ is not violated.
\end{proof}

\section*{Acknowledgment}
\ifthenelse{\boolean{anon}}{
	Anonymous acknowledgement.
}{
	The authors gratefully acknowledge financial support by the Horizon 2020 EU Framework Project CONCERT under grant 101016007 and financial support by the DAAD program Konrad Zuse Schools of Excellence in Artificial Intelligence, sponsored by the Federal Ministry of Education and Research.
}

\ifCLASSOPTIONcaptionsoff
  \newpage
\fi

\bibliographystyle{IEEEtranN}
\bibliography{library_cleaned}

\ifthenelse{\boolean{anon}}{
	\begin{IEEEbiography}{Anonymous Author}
		Biography text here.
	\end{IEEEbiography}
}{
	\begin{IEEEbiography}{Jakob Thumm}
		Biography text here.
	\end{IEEEbiography}
	\begin{IEEEbiographynophoto}{Leonardo Maglanoc}
	Biography text here.
	\end{IEEEbiographynophoto}
	
	\begin{IEEEbiographynophoto}{Julian Balletshofer}
	Biography text here.
	\end{IEEEbiographynophoto}

	\begin{IEEEbiographynophoto}{Matthias Althoff}
	Biography text here.
	\end{IEEEbiographynophoto}
}

\end{document}